\documentclass{article} 

\usepackage{preprint,times}
\iclrpreprintcopy

\usepackage{amsmath,amsfonts,bm}

\def\eqref#1{equation~\ref{#1}}

\def\1{\bm{1}}

\DeclareMathAlphabet{\mathsfit}{\encodingdefault}{\sfdefault}{m}{sl}
\SetMathAlphabet{\mathsfit}{bold}{\encodingdefault}{\sfdefault}{bx}{n}

\usepackage[utf8]{inputenc} 
\usepackage[T1]{fontenc}    
\usepackage{hyperref}       
\usepackage{url}            
\usepackage{booktabs}       
\usepackage{amsfonts}       
\usepackage{nicefrac}       
\usepackage{microtype}      
\usepackage{hyperref} 

\usepackage{booktabs}
\usepackage{adjustbox}
\usepackage{pdflscape}
\usepackage{rotating}

\usepackage{amsmath}
\usepackage{amssymb}
\usepackage{graphicx}
\usepackage{wrapfig}
\usepackage{sidecap}
\usepackage{algorithm2e}
\usepackage{multirow}
\RestyleAlgo{ruled}
\LinesNumbered
\DontPrintSemicolon
\usepackage[english]{babel}

\usepackage{float}
\usepackage{enumitem}
\usepackage{titlesec}
\usepackage{setspace}
\usepackage{enumitem}
\usepackage[table]{xcolor}
\usepackage{booktabs}
\setlist[itemize]{leftmargin=*}

\usepackage{amsthm}
\newtheorem{proposition}{Proposition}

\definecolor{first}{HTML}{D9EAD3}   
\definecolor{second}{HTML}{D9EAF7}  
\definecolor{third}{HTML}{FFF2CC}   
\newcommand{\first}[1]{\cellcolor{first}\textbf{#1}}
\newcommand{\second}[1]{\cellcolor{second}#1}
\newcommand{\third}[1]{\cellcolor{third}#1}

\title{Title}

\author{
\textbf{Quanling Zhao}$^1$ \quad
\textbf{Jiaying Yang}$^1$ \quad
\textbf{Ye Tian}$^1$ \quad
\textbf{Josh Victoria}$^1$ \quad
\textbf{Zhijun Wang}$^1$ \\
\textbf{Pietro Mercati}$^2$ \quad
\textbf{Onat Gungor}$^1$ \quad
\textbf{Tajana Rosing}$^1$ \\
$^1$Computer Science and Engineering, University of California San Diego \\
$^2$Intel Corporation \\
\texttt{\{quzhao,jiy018,yet002,jvictoria,zhw106,ogungor,tajana\}@ucsd.edu} \\
\texttt{pietro.mercati@intel.com}
}

\begin{document}

\title{RGLD: \underline{\textbf{R}}andomized \underline{\textbf{G}}lobal-\underline{\textbf{L}}ocal \underline{\textbf{D}}ensity Estimation for Tabular Anomaly Detection}

\maketitle

\begin{abstract}
Unsupervised tabular anomaly detection requires methods that are accurate, robust across heterogeneous datasets, and computationally efficient. Classical statistical detectors are often efficient, but they usually rely on a fixed data view and a single notion of abnormality. Deep anomaly detectors can learn more flexible scoring functions, but they are substantially slower and difficult to tune in unsupervised settings due to the lack of a reliable supervisory signal. We propose \textbf{RGLD}, a randomized global-local density estimator for efficient unsupervised tabular anomaly detection. RGLD combines a global random-feature density branch, which identifies samples in broadly low-density regions, with a local neighbor branch, which detects samples that are weakly supported by nearby observations. Both branches operate over feature-bagged randomized views, allowing RGLD to expose anomaly evidence that may be hidden in any single representation. We conduct experiments on 47 tabular datasets against 23 statistical and deep anomaly detection baselines under fully unsupervised setting. RGLD achieves the strongest dataset-level AUROC performance, ranking 1st in dataset wins, and ranks 2nd in AUPRC wins. RGLD is also faster than all evaluated deep detectors, achieving 50$\times$–580$\times$ speedups, and remains competitive with statistical methods in runtime, yielding a favorable accuracy-efficiency tradeoff.
\end{abstract}

\section{Introduction}

Unsupervised anomaly detection is a fundamental problem in machine learning with wide-ranging applications~\citep{han2022adbench,jiang2023adgym,gungor2024robust,chen2025pyod,chatterjee2022iot}. Given a dataset dominated by normal samples, the goal is to assign high anomaly scores to rare, irregular, or structurally inconsistent observations, without access to labels. As datasets continue to grow and anomaly detection is increasingly deployed in time-sensitive applications, efficiency and scalability become central requirements.

A long-standing principle in anomaly detection is that anomalies are associated with low-density or weakly supported regions of the data distribution~\citep{chandola2009anomaly,ruff2021unifying}. This idea motivates many statistical methods, including nearest-neighbor detectors~\citep{ramaswamy2000efficient}, local density estimators~\citep{breunig2000lof}, kernel methods~\citep{scholkopf2001estimating}, isolation-based methods~\citep{liu2008isolation}, and histogram-based approaches~\citep{goldstein2012histogram}. These methods are attractive because they are often simple, interpretable, and applicable without supervision. However, many statistical detectors are limited by two fixed choices: how abnormality is defined and where it is measured. First, each method typically adopts a particular notion of abnormality. For example, density-based methods treat anomalies as points in low-density regions, while nearest-neighbor methods treat anomalies as points with weak local support. These criteria are complementary: a point may be globally rare but not locally isolated, or locally isolated without being globally extreme. Second, a given notion of abnormality is often evaluated in a fixed view of the data, such as the full feature space, a chosen distance metric, and a scale. This can fail when the anomaly signal is visible only in a subset of features or under a different projection. In such cases, irrelevant or high-variance features can dilute the anomaly signal, making a point appear well supported in the full space even though it is clearly isolated in an informative subspace~\citep{aggarwal2001outlier,kriegel2009outlier}. Thus, robust tabular anomaly detection requires both diverse notions of abnormality and diverse views in which those notions are evaluated.

Deep anomaly detectors offer another route by learning nonlinear representations or anomaly scoring functions~\citep{ruff2018deep,zong2018deep,pang2021deep}. This flexibility can be valuable when anomalies are not well captured by a rigid notion of abnormality. However, in fully unsupervised settings, the same flexibility also creates a practical challenge as there is no reliable supervisory signal for training or choosing objectives. As a result, deep detectors often rely on proxy learning tasks whose effectiveness can vary substantially across datasets. They also introduce additional training cost and optimization variability, and benchmark studies have found that they can be orders of magnitude slower than lightweight statistical methods~\citep{han2022adbench,jiang2023adgym}.

This motivates a different question: can we obtain diverse multi-view anomaly signals without the computational burden of deep anomaly detectors? We argue that for tabular anomaly detection, a promising path is to revisit classical density principles through randomized computation. Instead of relying on a single view or a single notion of abnormality, we can generate many inexpensive randomized views of the data and aggregate their anomaly signals. Each view provides a partial estimate of abnormality, but the diversity across views and anomaly mechanisms can expose different anomalies and reduce dependence on any single notion of abnormality.

Built around this idea, we propose \textbf{RGLD: \underline{R}andomized \underline{G}lobal-\underline{L}ocal \underline{D}ensity Estimation for Tabular Anomaly Detection}, an efficient and scalable detector for fully unsupervised tabular anomaly detection. RGLD combines two complementary scoring branches. The global branch uses random-feature density estimation to identify samples in broadly low-density regions, capturing global rarity while avoiding expensive pairwise density computation. The local neighbor branch measures neighborhood support in randomized projected spaces, capturing locally isolated samples that may not be globally extreme. RGLD further ensembles these scores over feature-bagged randomized views, revealing anomalies hidden in the full feature space. Finally, rank-based aggregation combines the resulting heterogeneous scores into a stable anomaly ranking. This design yields a simple, parallelizable, and highly efficient unsupervised detector for settings where both detection accuracy and runtime matter.

\vspace{-2mm}
\begin{itemize}[itemsep=2pt]
    \item We propose \textbf{RGLD}, a fully unsupervised framework that turns classical density and neighborhood principles into scalable randomized estimators for tabular anomaly detection.
    \item We provide an analysis showing how feature-bagged views can strengthen sparse anomaly contrast, while random-feature density estimation and sampled-reference neighbor scoring preserve useful support signals without conventional quadratic pairwise costs.
    \item Finally, we evaluate RGLD on 47 ADBench datasets against 23 statistical and deep baselines under fully unsupervised setting. RGLD ranks 1st in AUROC wins and 2nd in AUPRC wins across datasets, while running 50$\times$--580$\times$ faster than evaluated deep anomaly detectors, yielding a favorable accuracy-efficiency tradeoff.
    
\end{itemize}
\vspace{-2mm}
More broadly, RGLD estimates two complementary density-derived signals (global KDE-like support and local neighborhood support) across feature-bagged randomized views, showing that simple statistical principles can remain highly competitive when redesigned around randomized computation.

\section{Related Work}

\textbf{Statistical and Randomized Anomaly Detection: } Statistical detectors define anomaly scores from explicit assumptions about distance, density, support, isolation, or distributional tails. Distance- and neighborhood-based methods such as $k$NN outlier detection~\citep{ramaswamy2000efficient}, LOF~\citep{breunig2000lof}, COF~\citep{tang2002enhancing}, and subspace outlier scoring~\citep{kriegel2009outlier} identify samples with weak local support or abnormal behavior in selected subspaces. Other methods rely on cluster structure, as in CBLOF~\citep{he2003discovering}, one-class support estimation, as in one-class SVM~\citep{scholkopf2001estimating}, or low-dimensional projection and reconstruction, as in PCA-based detection~\citep{shyu2003novel}. Efficient distributional detectors such as HBOS~\citep{goldstein2012histogram}, COPOD~\citep{li2020copod}, and ECOD~\citep{li2022ecod} score deviations from marginal, copula-based, or empirical feature distributions. Randomized and ensemble methods reduce dependence on a single detector or data view: Feature Bagging~\citep{lazarevic2005feature} ensembles detectors over feature subsets, Isolation Forest~\citep{liu2008isolation} uses randomized partitions, LODA~\citep{pevny2016loda} aggregates random-projection histograms, and ADERH~\citep{durani2026anomaly} builds an ensemble from randomized hypersphere pairs. Zimek et al.~\citep{zimek2014ensembles} provide a broader discussion of ensemble methods for unsupervised outlier detection. These methods demonstrate the strength of lightweight statistical scoring and randomization. However, most emphasize a primary scoring principle, such as local density, isolation, marginal tail behavior, or projection-based histograms. RGLD instead explicitly combines global density support and local neighborhood support, and evaluates both across feature-bagged randomized views and multiple density scales.

\textbf{Deep Unsupervised Anomaly Detection: } Deep anomaly detectors replace fixed statistical scores with learned representations or proxy objectives. Deep SVDD~\citep{ruff2018deep} learns a representation that concentrates training samples around a hypersphere center, while REPEN~\citep{pang2018learning} learns representations for random distance-based outlier scoring. Reconstruction- and density-based approaches include DAGMM~\citep{zong2018deep}, which combines an autoencoder with a Gaussian mixture density estimator, and RCA~\citep{liu2021rca}, which uses a collaborative autoencoder for anomaly detection. Other methods construct self-supervised objectives for tabular or general data: RDP~\citep{wang2019unsupervised} predicts random distances, GOAD~\citep{bergman2020classification} performs transformation classification, NeuTraL~\citep{qiu2021neural} learns neural transformations, ICL~\citep{shenkar2022anomaly} uses internal contrastive learning, and SLAD~\citep{xu2023fascinating} constructs scale-based supervisory signals. DIF~\citep{xu2023deep} further combines deep representations with isolation-based scoring. Although these methods can learn flexible anomaly scores, fully unsupervised tabular anomaly detection provides limited guidance for selecting architectures, proxy tasks, and optimization objectives, and introduces substantial computational cost. RGLD instead seeks flexible anomaly evidence through randomized statistical computation: it combines different anomaly signals and views without requiring a deep learned representation.

\section{Why Randomized Global-Local Density Estimation?}
\label{sec:motivation}
\begin{wrapfigure}{r}{0.5\textwidth}
    \centering
    \vspace{-5mm}
    \includegraphics[width=0.95\linewidth]{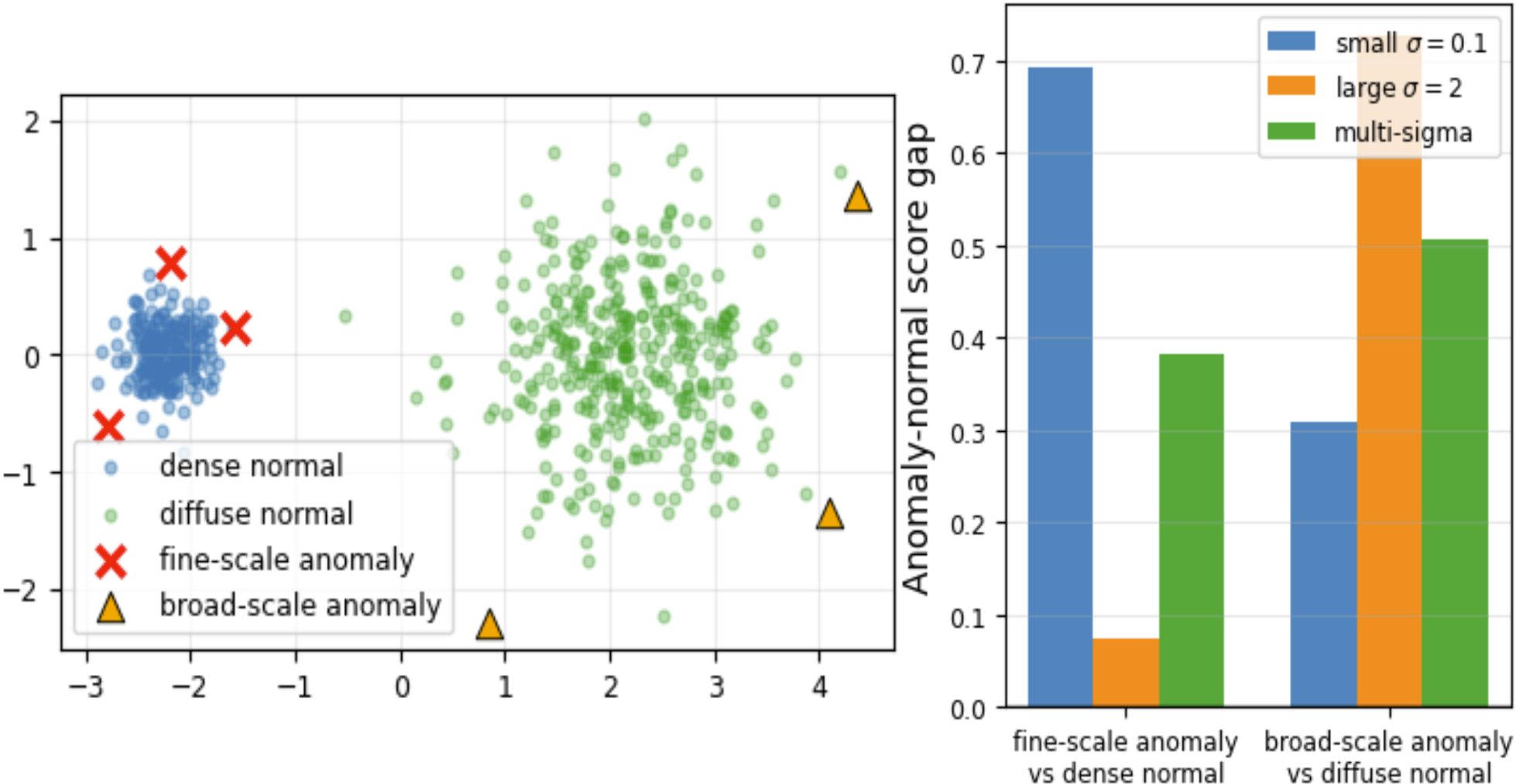}
    \vspace{-4mm}
    \caption{
    \textbf{Unknown anomaly scale.} \textbf{Left:} Fine-scale anomalies occur near a dense cluster, while broad-scale anomalies occur around a diffuse cluster. \textbf{Right:} Small and large bandwidths favor different anomaly scales, while multi-sigma scoring provides a more balanced signal when the correct scale is unknown.}
    \label{fig:motivation_scale}
\end{wrapfigure}
We first use simple synthetic examples to motivate why RGLD is built around randomized global-local density estimation. In fully unsupervised anomaly detection, labels are unavailable, so one of the most general principles is to measure how well each sample is supported by the observed data distribution. Under this view, normal samples tend to lie in high-support regions, while anomalies are weakly supported, either because they occur in globally low-density regions or because they are locally isolated from nearby samples~\citep{chandola2009anomaly,ruff2021unifying}. This support-based perspective is attractive for tabular data because it is label-free, model-agnostic, and closely connected to many statistical detectors. However, turning this principle into a robust detector is nontrivial: the relevant density scale is unknown, global and local notions of support capture different anomaly mechanisms, and the feature view in which the anomaly is visible is usually unknown. The examples below illustrate these three challenges and motivate the main design choices of RGLD.

\textbf{Unknown Anomaly Scale:} A central difficulty in density-based anomaly detection is that the appropriate density scale is unknown. As shown in Figure~\ref{fig:motivation_scale}, KDE-style density scoring requires a bandwidth $\sigma$ that controls the resolution of the density estimate~\citep{silverman2018density}. However, anomalies may appear at different resolutions. Fine-scale anomalies may occur near a dense normal cluster, while broad-scale anomalies may be unusual only relative to a more diffuse normal structure. A single bandwidth can therefore emphasize one anomaly type while suppressing another. This motivates the multi-$\sigma$ design of RGLD's global density branch, which aggregates density evidence across several bandwidths instead of committing to one fixed resolution.

\begin{figure}[t]
    \centering
    \includegraphics[width=0.98\linewidth]{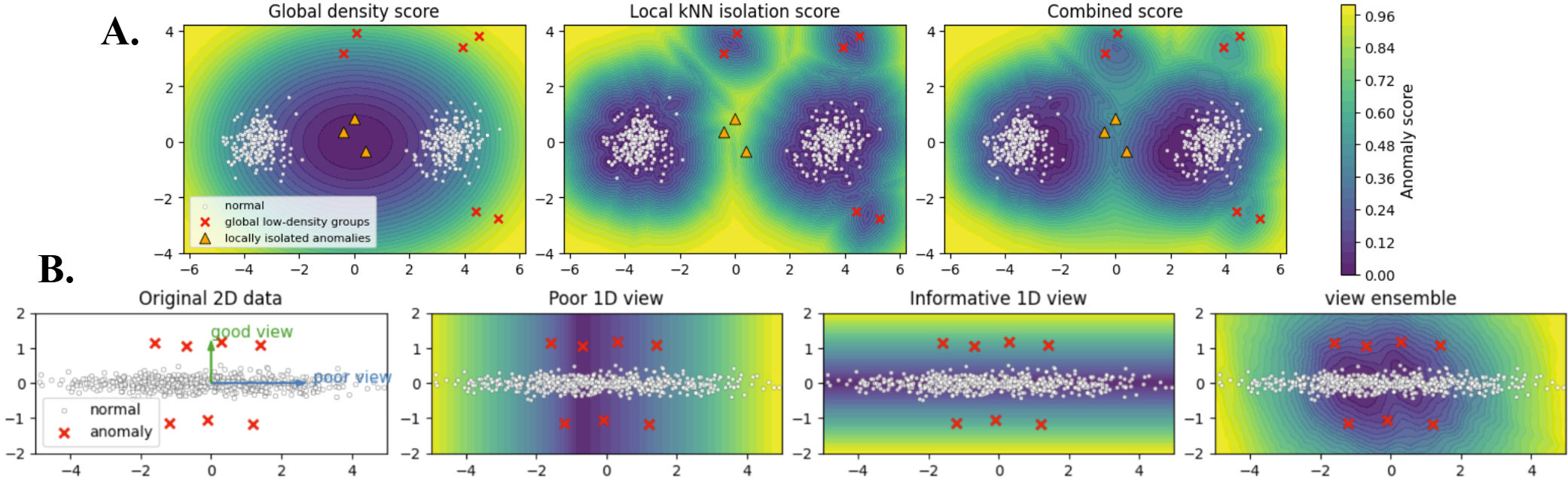}
    \vspace{-4mm}
    \caption{
    \textbf{Complementary evidence and view dependence.}
    \textbf{A.} Global density scoring detects broadly low-density anomalies, while local $k$NN isolation detects samples that are weakly supported by nearby observations. Their combination captures complementary anomaly mechanisms.
    \textbf{B.} Anomalies can be hidden in a poor view but become apparent in an informative projected view. Randomized view ensembling improves robustness by aggregating evidence across many perspectives.}
    \label{fig:motivation_global_local_view}
\end{figure}
\textbf{Global Rarity and Local Isolation Are Complementary:} A second challenge is that abnormality can take different forms. Some anomalies are globally rare and lie in broadly low-density regions of the data distribution. Other anomalies are locally isolated as they may not be globally extreme, but they are poorly supported relative to nearby samples. A detector based only on global density or only on local isolation can miss one of these cases. Here, local isolation is measured using a $k$NN isolation score~\citep{ramaswamy2000efficient}, which assigns higher anomaly scores to points with larger average distances to their k nearest neighbors. Figure~\ref{fig:motivation_global_local_view}A shows this complementarity. The global density score assigns high anomaly values to points far from the main distribution, capturing broad low-density structure. In contrast, the local $k$NN isolation score emphasizes points that are poorly supported by nearby samples, including anomalies between or near normal clusters.

\textbf{Anomalies Are View-Dependent:} A third challenge is that anomalies may be visible only under certain feature views. In tabular data, the relevant abnormal structure often lies in a subset of features or along a particular direction. Irrelevant features or high-variance nuisance directions can obscure the anomaly signal in the full space~\citep{aggarwal2001outlier,kriegel2009outlier}. Thus, a fixed representation can make an anomaly appear normal even when another view separates it clearly. Figure~\ref{fig:motivation_global_local_view}B illustrates this issue. In the original two-dimensional data, the anomaly signal depends strongly on the projection direction. A poor one-dimensional view largely hides the anomalies, while an informative view separates them more clearly. An ensemble over randomized views increases the chance that some views expose the relevant abnormal structure, even when the correct view is unknown in advance.

\section{RGLD Algorithm}
\label{sec:method}

The previous section motivates three design principles: collect density evidence across multiple scales, combine global rarity with local isolation, and evaluate anomaly scores across randomized views. We now describe \textbf{RGLD}, a fully unsupervised anomaly detector that implements these principles through randomized global-local density estimation. Figure~\ref{fig:overview} summarizes the RGLD pipeline. Given an unlabeled dataset $X=\{x_i\}_{i=1}^n \subset \mathbb{R}^d$, RGLD produces an anomaly score $s(x)$, where larger values indicate more anomalous samples. RGLD consists of two complementary branches. The global random-feature density branch estimates broad distributional support using randomized kernel density surrogates. The local neighbor branch estimates local support using sampled nearest-neighbor distances in randomized projected spaces. Following prior work on feature bagging for outlier detection~\citep{lazarevic2005feature}, randomized feature subsets increase view diversity and help expose subspace-specific anomalies. Both branches operate over independently sampled randomized feature subsets and projections, and their scores are combined through rank-based aggregation.

\textbf{Randomized Views:}
For each view $t$, RGLD samples a feature subset $S_t \subseteq \{1,\ldots,d\}$ and a random projection matrix $R_t \in \mathbb{R}^{|S_t|\times r}$, representing a sample $x$ as $z_t(x)=x_{S_t}R_t$. This feature-bagged projection serves two roles: feature subsets help expose subspace-specific anomalies that may be diluted in the full feature space, and random projections reduce dimensionality for high-dimensional data while approximately preserving the geometry needed for anomaly detection~\citep{johnson1984extensions,dasgupta2003elementary}. The global branch uses each view to define a random-feature density estimator, whereas the local neighbor branch uses independently sampled views to define local isolation estimators. Because any single view may be incomplete or noisy, RGLD aggregates scores across many views rather than relying on one fixed representation.
\begin{figure}[t]
    \centering
    \includegraphics[width=0.8\linewidth]{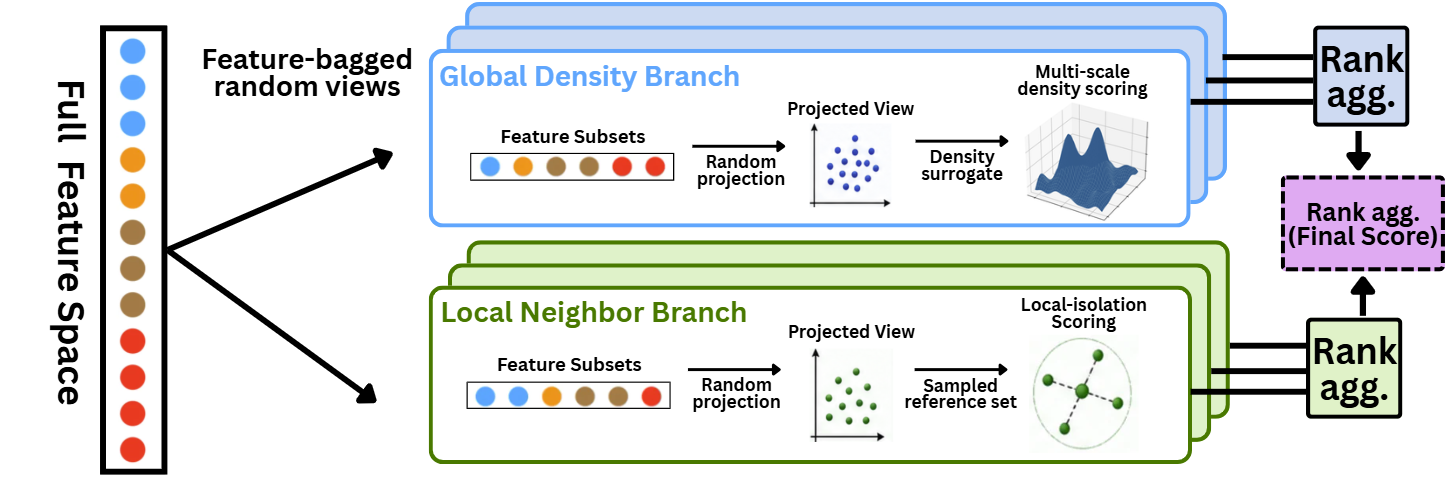}
    \vspace{-4mm}
    \caption{\textbf{RGLD Overview.} RGLD constructs feature-bagged random views and processes them through two complementary branches. The global density branch uses random projections and density surrogates for multi-scale density scoring, capturing global rarity. The local neighbor branch uses projected views and sampled reference sets to measure local isolation. Scores are rank-aggregated within each branch and then rank-aggregated again to produce the final anomaly score.}
    \label{fig:overview}
\end{figure}

\textbf{Global Density Branch:} The global branch measures whether a sample lies in a broadly supported region of the data distribution. A natural nonparametric density-based score is kernel density estimation (KDE)~\citep{rosenblatt1956remarks,parzen1962estimation}: for bandwidth $\sigma$, one may score a sample $x$ by $\widehat{p}_{\sigma}(x)=\frac{1}{n}\sum_{i=1}^n k_{\sigma}(x,x_i)$, where $k_{\sigma}$ is a Gaussian kernel. Low values of $\widehat{p}_{\sigma}(x)$ indicate weak distributional support and therefore high anomaly likelihood. However, evaluating this score for all samples requires pairwise kernel comparisons, leading to $O(n^2)$ cost per view and bandwidth. This becomes expensive when density scoring is repeated across many randomized views and multiple density scales. RGLD avoids this pairwise computation by replacing explicit KDE with random-feature density estimation. For a randomized view $t$, let $z_i^{(t)}=z_t(x_i)$. Random features approximate shift-invariant kernels by finite-dimensional inner products~\citep{rahimi2007random} $k_{\sigma}(z,z')\approx \langle \phi_{t,\sigma}(z),\phi_{t,\sigma}(z')\rangle$. Substituting this approximation into KDE yields the key density estimator:
\begin{equation*}
\small
    \widehat{p}_{t,\sigma}(x)
    \approx \widetilde{p}_{t,\sigma}(x) = 
    \left\langle
        \phi_{t,\sigma}(z_t(x)),
        \mu_{t,\sigma}
    \right\rangle,
    \qquad
    \mu_{t,\sigma}
    =
    \frac{1}{n}
    \sum_{i=1}^{n}
    \phi_{t,\sigma}(z_i^{(t)})
\end{equation*}
Thus, the full training set is compressed into one mean embedding $\mu_{t,\sigma}$~\citep{muandet2016kernel} per randomized view and bandwidth. Scoring a sample only requires an inner product with this stored vector, rather than comparisons to all training samples, reducing the complexity of density estimation over all data points from $O(n^2)$ to $O(n)$. This mean-embedding form is the main efficiency gain of the global branch. To achieve the kernel sum approximation required for KDE, RGLD uses cosine random features. For a view $t$, let $W_t$ be a random frequency matrix, $b_t$ a random phase vector, and $D$ the random-feature dimension. The random frequencies are sampled using quasi-Monte Carlo (QMC)~\citep{yang2014quasi} to reduce approximation variance at small $D$. Then, the random-feature map at bandwidth $\sigma$ is:
\begin{equation*}
\small
    \phi_{t,\sigma}(z)
    =
    \sqrt{\frac{2}{D}}
    \cos\left(
        \frac{W_t z}{\sigma} + b_t
    \right)
\end{equation*}
This form also makes multi-scale density scoring cheap. For each randomized view, RGLD first computes the linear random-feature projection $P_t(z)=W_tz$. Different bandwidths only rescale this same projection inside the cosine map. Therefore, for bandwidths $\sigma_{t,m}=\alpha_m\widehat{\sigma}_t$, where $\widehat{\sigma}_t$ is a base bandwidth estimated from sampled projected distances, RGLD reuses $P_t(z)$ and changes only $P_t(z)/\sigma_{t,m}$. The resulting density support score is:
\begin{equation*}
\small
    d_{t,m}(x)
    =
    \left\langle
        \sqrt{\frac{2}{D}}
        \cos\left(
            \frac{P_t(z_t(x))}{\sigma_{t,m}} + b_t
        \right),
        \mu_{t,m}
    \right\rangle,
    \qquad
    \sigma_{t,m}=\alpha_m\widehat{\sigma}_t 
\end{equation*}
The support score $d_{t,m}(x)$ is large when $x$ is well supported under view $t$ and scale $m$. RGLD converts it to an anomaly score by negation, $a_{t,m}^{\mathrm{glob}}(x)=-d_{t,m}(x)$. Within each global view, scores from different bandwidths can have different scales. RGLD therefore standardizes the negative density scores for each bandwidth and averages them:
\begin{equation*}
\small
    a_t^{\mathrm{glob}}(x)
    =
    \frac{1}{M}
    \sum_{m=1}^{M}
    \mathcal{Z}
    \left(
        a_{t,m}^{\mathrm{glob}}(x)
    \right)
\end{equation*}
where $\mathcal{Z}$ denotes standardization across samples for a fixed bandwidth. The view-level scores $a_t^{\mathrm{glob}}$ are then rank-aggregated across randomized global views to obtain $s_{\mathrm{glob}}$. In this way, the global branch acts as an efficient multi-view, multi-scale approximation to KDE.

\textbf{Local Neighbor Branch:} The global branch captures broad low-density structure, but some anomalies are better characterized by local isolation. RGLD, therefore, includes a local neighbor branch that measures whether a sample is weakly supported by nearby reference points in randomized projected spaces. For each local view $t$, RGLD samples a feature subset and random projection, producing $z_i^{(t)}=z_t(x_i)$. The reference set $C_t$ is sampled uniformly without replacement from the projected training samples $\{z_i^{(t)}\}_{i=1}^n$. We use $\mathcal{N}_k(z_t(x);C_t)$ to denote the set of the k nearest reference points to $z_t(x)$ within $C_t$. For a sample $x$, the projected-neighbor score is the root mean squared distance from $z_t(x)$ to its k nearest points in $C_t$:
\begin{equation*}
\small
    a_t^{\mathrm{loc}}(x)
    =
    \left(
    \frac{1}{k}
    \sum_{u\in\mathcal{N}_k(z_t(x);C_t)}
    \|z_t(x)-u\|_2^2
    \right)^{1/2}
\end{equation*}
Larger values indicate weaker local support and therefore higher anomaly likelihood. This branch differs from a standard full-space $k$NN detector~\citep{ramaswamy2000efficient} in three ways. First, distances are computed in feature-bagged projected spaces rather than in the original feature space. Second, each local view uses a sampled reference set rather than the full training set, reducing the cost of neighbor scoring. Third, RGLD aggregates local isolation scores across randomized views, reducing dependence on any single projection, feature subset, or reference sample. The local neighbor branch score $s_{\mathrm{loc}}$ is obtained by rank-aggregating $a_t^{\mathrm{loc}}$ across local views.

\textbf{Rank-Based Aggregation:} The global and local branches produce scores with different meanings and expose anomalies of different mechanisms. The global branch produces negative density-support scores, while the local neighbor branch produces distance-based isolation scores. Directly averaging raw scores can therefore be unstable. RGLD instead aggregates scores by rank. For a score vector $v\in\mathbb{R}^n$, RGLD maps each score to its normalized rank $\rho_i(v) = \frac{\operatorname{rank}(v_i)-1}{n-1}$ where larger ranks correspond to more anomalous samples. Given score vectors $v^{(1)},\ldots,v^{(T)}$, rank aggregation is
\begin{equation*}
\small
    \operatorname{RankAgg}_i
    \left(
        v^{(1)},\ldots,v^{(T)}
    \right)
    =
    \frac{1}{T}
    \sum_{t=1}^{T}
    \rho_i(v^{(t)})
\end{equation*}
This converts heterogeneous scores into a common scale and emphasizes consistency in anomaly ranking rather than raw score magnitude. The global branch score is $s_{\mathrm{glob}}$, and the local neighbor branch score is $s_{\mathrm{loc}}$. The final RGLD score is:
\begin{equation*}
\small
\begin{split}
    &s(x_i) =\operatorname{RankAgg}_i \left( s_{\mathrm{glob}}, s_{\mathrm{loc}} \right) \\
    \text{{where}}  \quad s_{\mathrm{glob}}(x_i)=\operatorname{RankAgg}_i&(a_1^{\mathrm{glob}},\ldots,a_{T_g}^{\mathrm{glob}}), \quad
    s_{\mathrm{loc}}(x_i)=\operatorname{RankAgg}_i(a_1^{\mathrm{loc}},\ldots,a_{T_\ell}^{\mathrm{loc}})
\end{split}
\end{equation*}
A sample receives a high RGLD score if it is consistently low-density across randomized global density views, locally isolated across projected-neighbor views, or both. This final aggregation treats global rarity and local isolation as complementary evidence while remaining insensitive to the score-scale mismatch between the two branches. Complete pseudo-code and implementation details are provided in Appendix~\ref{app:pseudocode}.

\section{Analysis: Subspace Contrast, Approximation, and Complexity}
\label{sec:analysis}

RGLD relies on randomized views and efficient approximations to obtain global-local anomaly scores. We provide an analysis of why feature-bagged views can increase anomaly contrast when the anomaly signal is sparse, and bound the approximation gaps introduced by the two main computational shortcuts: random-feature global density estimation and sampled-reference local neighbor scoring. These results are not intended as end-to-end anomaly detection guarantees, which would require assumptions on the anomaly-generating process. Instead, they clarify the conditions under which RGLD preserves useful density and locality signals while avoiding quadratic computation.

\textbf{Feature-Bagged Subspace Contrast: }We first formalize the mechanism by which feature-bagged views used by RGLD can yield stronger anomaly contrast than the full feature space when the anomaly 
signal is concentrated in a small set of informative coordinates. Consider an anomaly-normal pair $(x_a,x_n)$. Suppose the anomaly signal lies in an unknown informative feature set $S^\star \subseteq [d]$ with $|S^\star|=s$, while the remaining features are nuisance coordinates.
\begin{proposition}[Feature-bagged subspaces can improve anomaly contrast]
\label{prop:feature-bagging-contrast}
Let $\mu_{\mathrm{sig}}$ and $\mu_{\mathrm{noise}}$ denote the average squared anomaly-normal separation on informative and nuisance coordinates, respectively, with $\mu_{\mathrm{sig}}>\mu_{\mathrm{noise}}$. For a feature subset $S\subseteq[d]$ of size $m$, define $\pi(S)=\frac{|S\cap S^\star|}{m}, C(S)=\frac{1}{m}\mathbb{E}\|(x_a-x_n)_S\|_2^2$. If $\pi(S)>\frac{s}{d}$, then $C(S)>C([d])$. i.e., has a larger average anomaly contrast than full space. Furthermore, $\Pr\left(\pi(S)>\frac{s}{d}\right)=\Pr\left(|S\cap S^\star|>\frac{ms}{d}\right)$. \textit{(See Appendix~\ref{app:prop1} for proof)}
\end{proposition}
Proposition~\ref{prop:feature-bagging-contrast} motivates the randomized-view design of RGLD. A fixed full-space representation can dilute sparse anomaly signals with nuisance features, while feature bagging increases the chance of seeing a view where the relevant coordinates are more concentrated. Random projection then provides an efficient alternative geometry within each selected subspace, preserving useful distance-based density and neighborhood signals while reducing cost.

\textbf{Global Density Approximation Gap: }Here, we analyze the approximation gap introduced by the global random-feature density branch. The method section defines the exact projected KDE score $\widehat{p}_{t,\sigma}(x)$ and its random-feature approximation $\widetilde{p}_{t,\sigma}(x)$. Here, we ask how far the surrogate can deviate from the exact KDE score for a fixed randomized view $t$ and bandwidth $\sigma$. The following proposition is a direct consequence of the classical random Fourier feature approximation bound~\citep{rahimi2007random,sutherland2015error}:
\begin{proposition}[RFF approximation of KDE scores]
\label{prop:rf_kde}
Fix a randomized view $t$, bandwidth $\sigma$, and training set $\{z_i^{(t)}\}_{i=1}^{N}$. Then for any fixed evaluation point $x$, with probability at least $1-\delta$,
\begin{equation*}
\small
    \left|
    \widetilde{p}_{t,\sigma}(x)
    -
    \widehat{p}_{t,\sigma}(x)
    \right|
    \le
    O\left(
    \sqrt{\frac{\log(1/\delta)}{D}}
    \right) \quad \text{\textit{(See Appendix~\ref{app:prop2} for proof})}
\end{equation*}
\end{proposition}
Proposition~\ref{prop:rf_kde} shows that the global branch remains close to the exact KDE score in each fixed view and bandwidth, with error decreasing as $D^{-1/2}$. Thus, RGLD can use many inexpensive random-feature density estimators while retaining a controlled approximation to KDE density scoring.

\textbf{Local Support Approximation Gap:}
We next analyze the approximation introduced by the local neighbor branch. In a fixed projected view $t$, full $k$NN scoring compares each point against all $n$ projected training samples. RGLD instead samples a reference set $C_t$ and computes $k$NN distances only against this subset. This reduces computation but introduces a sampling gap: nearby training samples may be missed if they are not included in $C_t$.

\begin{proposition}[Reference coverage for local support]
\label{prop:kNN_reference}
Fix a projected view $t$, an evaluation point $x$, and radius $\rho>0$. Let 
$B_t(x,\rho)=\{z_i^{(t)}:\|z_t(x)-z_i^{(t)}\|_2\le \rho\}$ be the set of projected training samples within radius $\rho$, and let $q_t(x,\rho)=|B_t(x,\rho)|/n$ be its empirical mass. If $C_t$ is sampled uniformly without replacement from the projected training samples, then
\begin{equation*}
\small
    \Pr\left[
    C_t \cap B_t(x,\rho)=\emptyset
    \right]
    \le
    \exp\left(-|C_t| q_t(x,\rho)\right) \quad \text{\textit{(See Appendix~\ref{app:prop3} for proof)}}
\end{equation*}
Equivalently, with probability at least $1-\exp(-|C_t| q_t(x,\rho))$, the reference set contains at least one point within distance $\rho$ of $z_t(x)$.
\end{proposition}

Proposition~\ref{prop:kNN_reference} shows that the probability of missing a local neighborhood decreases exponentially with the reference size $|C_t|$ and the empirical local mass $q_t(x,\rho)$. Thus, sampled-reference $k$NN preserves local support when nearby samples are not too rare, while assigning large scores to points with weak local support. For $k>1$, the same argument extends by considering the probability that $C_t$ contains at least $k$ points in $B_t(x,\rho)$, which follows a hypergeometric tail bound~\citep{hoeffding1963probability,serfling1974probability}. Aggregating across multiple projected-neighbor views and independently sampled reference sets further reduces dependence on any single reference sample.

\textbf{Computational Complexity: } Let $n$ be the number of samples, $d_t$ the number of selected features in a view, $r$ the projection dimension, $D$ the random-feature dimension, $M$ the number of bandwidths, $T_g$ the number of global views, $T_\ell$ the number of projected-neighbor views, and $|C_t|$ the reference-set size. For one global view, exact multi-scale KDE requires $O(Mn^2r)$ kernel computations. RGLD replaces this with random-feature estimation: feature-bagged projection costs $O(nd_t r)$, the random-feature projection costs $O(nrD)$, and evaluating $M$ bandwidths costs $O(MnD)$. Across global views, the cost is $O\left(T_g(nd_t r + nrD + MnD)\right)$. The key saving is avoiding the $n^2$ KDE term. For projected-neighbor views, full $k$NN costs $O(n^2r)$. RGLD instead compares each point to a sampled reference set of size $|C_t|$, giving $O\left(T_\ell(nd_t r+n|C_t|r)\right)$ across local views. Rank aggregation adds $O(T_g n\log n + T_\ell n\log n)$ sorting cost. Overall, RGLD replaces the quadratic costs of exact KDE and full $k$NN with near linear-in-$n$ randomized approximations, up to rank-sorting overhead, controlled by $D$, $|C_t|$, $r$. Since randomized views are independent, both branches are naturally parallelizable.

\section{Experimental Analysis}
\label{sec:experiments}

\textbf{Experimental Setup: } We evaluate RGLD on 47 tabular datasets from ADBench~\citep{han2022adbench}, which span diverse real-world anomaly detection settings including healthcare, finance, network/web security, image-derived features, documents, and scientific measurements. We compare against 23 baselines from ADBench~\citep{han2022adbench}, DeepOD~\citep{xu2023deep,xu2024calibrated}, and the recently proposed ADERH~\citep{durani2026anomaly}, covering both statistical and deep anomaly detectors. All methods are evaluated in the fully unsupervised setting without using labels for model selection or hyperparameter tuning. Baselines use the official ADBench/DeepOD default configurations. We report Area Under the Receiver Operating Characteristic Curve (AUROC), Area Under the Precision-Recall Curve (AUPRC), and average runtime across datasets. Statistical methods are run on an Intel Core i7-155H CPU, while deep methods are run on an NVIDIA RTX 4070 Laptop GPU. RGLD uses one fixed default configuration across all datasets. The global random-feature density branch uses $T_g=40$ randomized views, QMC random features with dimension $D=10$, median-distance bandwidth estimation from $8,192$ sampled pairs, and multi-scale bandwidth multipliers $\{0.25,0.5,1.0,2.0\}$. The local neighbor branch uses $T_\ell=10$ randomized views, reference set size $400$, and sampled reference size $|C_t|=400$ and $k=10$. For both the global and local neighbor branches, each randomized view samples $30\%$ of the input features before applying a random projection. Full RGLD configurations are provided in Appendix~\ref{app:hyperparameter}. 

\textbf{Anomaly Detection Results: }
\begin{table}[t]
\caption{
\textbf{Main benchmark results.} On 47 ADBench tabular datasets across RGLD and 23 statistical and deep learning methods, all evaluated under the same fully unsupervised protocol. \fcolorbox{black}{first}{\rule{0pt}{3pt}\rule{3pt}{0pt}} \fcolorbox{black}{second}{\rule{0pt}{3pt}\rule{3pt}{0pt}} \fcolorbox{black}{third}{\rule{0pt}{3pt}\rule{3pt}{0pt}} Marks 1st, 2nd and 3rd place respectively.}
\centering
\tiny
\setlength{\tabcolsep}{4pt}
\begin{tabular}{l l c c c c c c c}
\toprule
Method & Type & AUROC & AUROC Wins  & AUPRC & AUPRC Wins & Avg. Rank & Time (s) & Relative Speed\\
\midrule
RGLD\textbf{*} & Statistical & \second{77.50} & \first{8} & \third{38.59} & \second{5} & \second{9.72} & 0.22 & 1\\
SLAD~\citep{xu2023fascinating} & DL & 73.39 & \second{7} & 35.23 & \first{10} & \third{9.76} & 39.78 & 0.01\\
ADERH~\citep{durani2026anomaly} & Statistical & \first{78.18} & \third{4} & \first{39.72} & 3 & \first{8.22} & 1.06 & 0.21\\
REPEN~\citep{pang2018learning} & DL & 74.72 & \third{4} & 37.44 & \third{4} & 10.52 & 47.88 & 0.005\\
COPOD~\citep{li2020copod} & Statistical & 74.39 & \third{4} & 35.22 & 3 & 11.52 & \third{0.07} & \third{3.3}\\
NeuTraL~\citep{qiu2021neural} & DL & 70.76 & \third{4} & 28.12 & \third{4} & 12.04 & 54.21 & 0.004 \\
$k$NN~\citep{ramaswamy2000efficient} & Statistical & 69.87 & \third{4} & 31.54 & 2 & 11.96 & 0.84 & 0.26\\
ICL~\citep{shenkar2022anomaly} & DL & 58.85 & \third{4} & 21.97 & 3 & 16.63 & 128.95 & 0.002\\
PCA~\citep{shyu2003novel} & Statistical & 73.60 & 3 & 38.27 & 1 & 10.92 & 0.66 & 0.33\\
HBOS~\citep{goldstein2012histogram} & Statistical & 74.04 & 3 & 36.96 & 2 & 10.84 & \first{0.04} & \first{5.41}\\
GOAD~\citep{bergman2020classification} & DL & 68.73 & 2 & 33.25 & 3 & 12.54 & 80.31 & 0.003\\
IForest~\citep{liu2008isolation} & Statistical & \third{76.17} & 1 & \second{38.78} & 0 & 9.91 & 0.52 & 0.42\\
DIF~\citep{xu2023deep} & DL & 76.13 & 1 & 36.62 & 0 & 10.43 & 15.88 & 0.01\\
CBLOF~\citep{he2003discovering} & Statistical & 74.47 & 1 & 37.32 & 2 & 10.45 & 0.30 & 0.73\\
ECOD~\citep{li2022ecod} & Statistical & 73.95 & 1 & 35.80 & 2 & 10.84 & 0.10 & 2.26\\
RDP~\citep{wang2019unsupervised} & DL & 74.15 & 1 & 30.99 & 2 & 11.43 & 29.41 & 0.01\\
OCSVM~\citep{scholkopf2001estimating} & Statistical & 69.98 & 1 & 33.23 & 1 & 13.63 & 1.80 & 0.12\\
LOF~\citep{breunig2000lof} & Statistical & 61.50 & 1 & 21.47 & 2 & 15.57 & 0.39 & 0.56\\
COF~\citep{tang2002enhancing} & Statistical & 62.75 & 1 & 23.07 & 0 & 15.98 & 18.53 & 0.01\\
DeepSVDD~\citep{ruff2018deep} & DL & 53.61 & 1 & 18.56 & 1 & 18.28 & 11.23 & 0.02\\
RCA~\citep{liu2021rca} & DL & 72.82 & 0 & 34.00 & 1 & 12.49 & 39.85 & 0.01\\
SOD\citep{kriegel2009outlier} & Statistical & 68.73 & 0 & 26.19 & 0 & 13.78 & 2.64 & 0.08\\
LODA~\citep{pevny2016loda} & Statistical & 64.93 & 0 & 30.52 & 1 & 15.29 & \second{0.05} & \second{4.84}\\
DAGMM~\citep{zong2018deep} & DL & 64.51 & 0 & 24.58 & 3 & 16.07 & 13.25 & 0.02\\
\bottomrule
\end{tabular}
\vspace{-2mm}
\label{tab:main_results}
\end{table}
Table~\ref{tab:main_results} summarizes the main results on 47 ADBench tabular datasets against 23 baselines. RGLD shows strong robustness across datasets, obtaining the most AUROC wins with $8$ first-place results. It also ranks second in mean AUROC with $77.50$, third in mean AUPRC with $38.59$, and second in AUPRC wins with $5$ wins. These results show that RGLD is consistently competitive across both ranking-based anomaly detection metrics. RGLD is also substantially more efficient than deep anomaly detectors, achieving 50$\times$–580$\times$ speedups. While ADERH obtains the best mean AUROC and AUPRC, RGLD is approximately $5\times$ faster and achieves twice as many AUROC wins. Compared with SLAD, which has the most AUPRC wins, RGLD is roughly two orders of magnitude faster. Although a few simple statistical methods are faster, they generally sacrifice accuracy and dataset-level wins. Overall, RGLD provides one of the strongest accuracy-efficiency tradeoffs among all evaluated methods, combining top-tier detection performance with lightweight runtime. Because ADERH is the strongest baseline in mean AUROC and AUPRC, we provide a detailed paired comparison in Appendix~\ref{app:aderh_comparison}. The comparison shows that RGLD is statistically competitive with ADERH in AUROC and AUPRC, while being significantly faster.

\textbf{Accuracy-Efficiency Tradeoff: }Table~\ref{tab:main_results} shows that RGLD achieves top-tier detection accuracy. We next examine its accuracy-efficiency tradeoff. Figure~\ref{fig:pareto} plots mean AUROC and mean AUPRC against average runtime across the 47 datasets, with runtime on a logarithmic scale. The figure compares the baselines with two RGLD variants (RGLD-Fast variant uses a reduced number of views). RGLD occupies a favorable accuracy-efficiency region in both metrics. In AUROC, RGLD performs close to the strongest method, ADERH, while running substantially faster. RGLD-Fast further shifts the tradeoff toward lower runtime while retaining high AUROC. In contrast, deep methods such as SLAD, REPEN, NeuTraL, GOAD, and ICL require much longer runtime while giving lower or comparable AUROC. Fast statistical methods are efficient, but many fall below RGLD in accuracy. The AUPRC plot shows a similar pattern, where RGLD remains among the top-performing methods and ahead of most statistical and deep baselines, while being highly efficient. RGLD-Fast again provides a faster variant with only a modest AUPRC reduction. These results support RGLD's central efficiency claim: randomized global-local density estimation can achieve competitive detection accuracy without the high training cost of deep anomaly detectors.
\begin{figure}[t]
    \centering
    \includegraphics[width=0.85\linewidth]{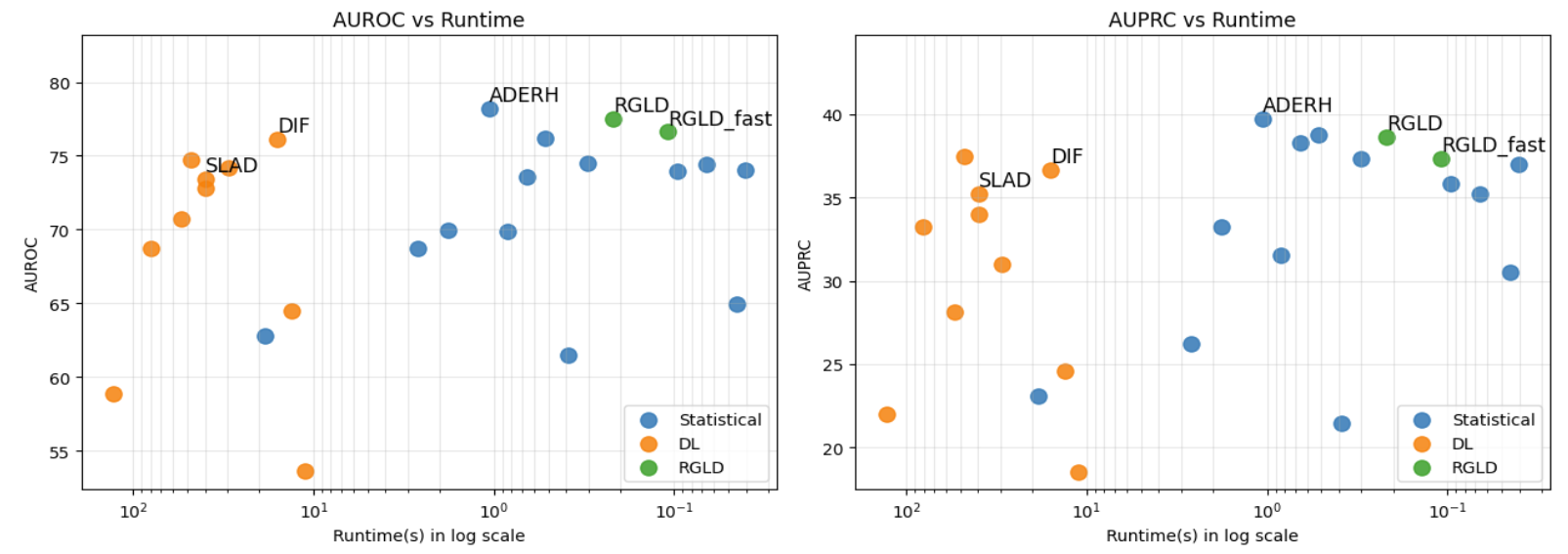}
    \vspace{-4mm}
    \caption{
    \textbf{Accuracy-efficiency tradeoff.} Mean AUROC and AUPRC versus average runtime across 47 ADBench tabular datasets. Runtime is shown on a logarithmic scale. For the RGLD-fast variant, we reduced the number of views, random feature dimension, and $k$NN reference set to half of the default RGLD. The full Pareto front graph can be found in Appendix~\ref{app:frontier}}
    \label{fig:pareto}
\end{figure}

\textbf{Ablation Study on Different RGLD Components: }
\begin{table}[t]
\caption{
\textbf{Ablation study on different RGLD components.} We report mean AUROC, mean AUPRC, average runtime, and changes relative to full RGLD. Removing any major component degrades performance, with the largest drop coming from removing the randomized-view ensemble.}
\centering
\footnotesize
\setlength{\tabcolsep}{5pt}
\begin{tabular}{l c c c c c c}
\toprule
Variant & AUROC $\uparrow$ & AUPRC $\uparrow$ & Time (s) $\downarrow$ 
& $\Delta$AUROC & $\Delta$AUPRC & $\Delta$Time \\
\midrule
RGLD & \textbf{77.50} & \textbf{38.59} & 0.22 & - & - & - \\
w/o multi-$\sigma$ & 77.24 & 38.46 & 0.19 & $-0.25$ & $-0.13$ & $-0.02$ \\
w/o feature bagging & 77.02 & 36.07 & 0.22 & $-0.48$ & $-2.51$ & $0.00$ \\
w/o local-neighbor & 75.52 & 38.07 & 0.15 & $-1.98$ & $-0.51$ & $-0.07$ \\
w/o view ensemble & 73.16 & 32.29 & \textbf{0.06} & $-4.33$ & $-6.30$ & $-0.16$ \\
\bottomrule
\end{tabular}
\vspace{-2mm}
\label{tab:ablation}
\end{table}
Table~\ref{tab:ablation} studies the contribution of RGLD's main components. The full model achieves the best AUROC and AUPRC, indicating that the proposed components are complementary rather than redundant. Removing multi-scale density scoring causes a modest drop, suggesting that aggregating multiple bandwidths improves robustness to unknown anomaly scale. Removing feature bagging leads to a larger AUPRC degradation, supporting the importance of subspace diversity. The local neighbor branch also contributes substantially: removing it reduces AUROC by $1.98$, showing that local isolation provides useful information beyond global density estimation. The largest degradation occurs when removing the randomized-view ensemble, which reduces AUROC by $4.33$ and AUPRC by $6.3$. Although this variant is faster, the large accuracy drop confirms the central design principle of RGLD: many inexpensive randomized views are essential for producing a robust anomaly ranking.

\textbf{Sensitivity to Ensemble Size and Random-Feature Dimension: }
\begin{figure}[t]
    \centering
    \includegraphics[width=\linewidth]{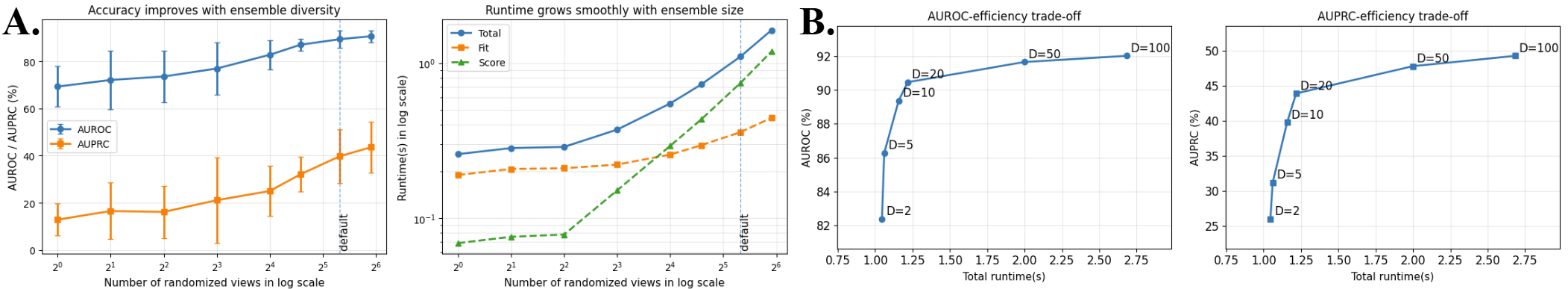}
    \vspace{-6mm}
    \caption{
    \textbf{Sensitivity to randomized views and random-feature dimension.}
    \textbf{A.} Increasing the number of randomized views improves AUROC and AUPRC for sparse subspace anomalies.
    \textbf{B.} Increasing the random-feature dimension $D$ improves accuracy with diminishing returns, illustrating the accuracy-efficiency tradeoff of the global density branch.
    }
    \label{fig:sensitivity}
\end{figure}
To isolate the mechanisms behind RGLD, we conduct controlled synthetic experiments in addition to the ADBench. The synthetic data contain sparse subspace anomalies: normal samples are drawn from an isotropic Gaussian distribution, while anomalies are shifted along a small subset of informative dimensions. This setup allows us to test whether increasing the number of randomized views helps expose subspace-specific anomaly signals. Figure~\ref{fig:sensitivity} studies two key efficiency parameters in RGLD: the number of randomized views and the random-feature dimension $D$. Panel A shows that increasing the number of randomized views improves both AUROC and AUPRC, indicating that the ensemble contributes complementary anomaly evidence rather than merely adding redundant detectors. Runtime grows smoothly with the number of views, and the default configuration lies near the high-performance region of the curve. This supports the central design principle of RGLD: aggregating many inexpensive randomized views yields a stronger anomaly ranking. Panel B studies the accuracy-efficiency tradeoff controlled by the random-feature dimension $D$. Larger $D$ improves both AUROC and AUPRC, consistent with the random-feature approximation analysis in Section~\ref{sec:analysis}. The improvement is largest when increasing $D$ from very small values and then gradually saturates, indicating diminishing returns at larger feature dimensions. Together, these results show that RGLD provides a controllable accuracy-efficiency tradeoff through both the number of randomized views and the random-feature dimension.


\section{Conclusion}
We introduced \textbf{RGLD}, a randomized global-local density estimator for fully unsupervised tabular anomaly detection. RGLD estimates complementary global and local anomaly signals from feature-bagged randomized views, allowing it to capture multiple anomaly mechanisms with efficiency. Empirically, RGLD achieves a strong accuracy-efficiency tradeoff on 47 ADBench tabular datasets against 23 statistical and deep baselines under a fully unsupervised setting, ranking 1st in AUROC wins, and 2nd in AUPRC wins while running 50$\times$--580$\times$ faster than evaluated deep anomaly detectors. Ablations show that randomized views, feature bagging, multi-scale density scoring, and projected-neighbor scoring each contribute to performance. These results suggest that the gains come not from a single component, but from combining many weak randomized density and locality signals into a stable ensemble. More broadly, RGLD shows that well-designed statistical methods can remain highly competitive for tabular anomaly detection when equipped with the right randomized inductive biases.

\section*{Acknowledgment}
This work was supported in part by PRISM and CoCoSys, centers in JUMP 2.0, an SRC program sponsored by DARPA (SRC grant number - 2023-JU-3135). This work was also supported by NSF grants \#2003279, \#1911095, \#2112167, \#2052809, \#2112665, \#2120019, \#2211386.

\newpage
\bibliography{iclr2027_conference}

@inproceedings{breunig2000lof,
  title={LOF: identifying density-based local outliers},
  author={Breunig, Markus M and Kriegel, Hans-Peter and Ng, Raymond T and Sander, J{\"o}rg},
  booktitle={Proceedings of the 2000 ACM SIGMOD international conference on Management of data},
  pages={93--104},
  year={2000}
}

@inproceedings{gungor2024robust,
  title={A robust framework for evaluation of unsupervised time-series anomaly detection},
  author={Gungor, Onat and Rios, Amanda and Mudgal, Priyanka and Ahuja, Nilesh and Rosing, Tajana},
  booktitle={International Conference on Pattern Recognition},
  pages={48--64},
  year={2024},
  organization={Springer}
}

@inproceedings{tang2002enhancing,
  title={Enhancing effectiveness of outlier detections for low density patterns},
  author={Tang, Jian and Chen, Zhixiang and Fu, Ada Wai-Chee and Cheung, David W},
  booktitle={Pacific-Asia conference on knowledge discovery and data mining},
  pages={535--548},
  year={2002},
  organization={Springer}
}

@article{he2003discovering,
  title={Discovering cluster-based local outliers},
  author={He, Zengyou and Xu, Xiaofei and Deng, Shengchun},
  journal={Pattern recognition letters},
  volume={24},
  number={9-10},
  pages={1641--1650},
  year={2003},
  publisher={Elsevier}
}

@article{scholkopf2001estimating,
  title={Estimating the support of a high-dimensional distribution},
  author={Sch{\"o}lkopf, Bernhard and Platt, John C and Shawe-Taylor, John and Smola, Alex J and Williamson, Robert C},
  journal={Neural computation},
  volume={13},
  number={7},
  pages={1443--1471},
  year={2001},
  publisher={MIT Press}
}

@inproceedings{liu2008isolation,
  title={Isolation forest},
  author={Liu, Fei Tony and Ting, Kai Ming and Zhou, Zhi-Hua},
  booktitle={2008 eighth ieee international conference on data mining},
  pages={413--422},
  year={2008},
  organization={IEEE}
}

@inproceedings{lazarevic2005feature,
  title={Feature bagging for outlier detection},
  author={Lazarevic, Aleksandar and Kumar, Vipin},
  booktitle={Proceedings of the eleventh ACM SIGKDD international conference on Knowledge discovery in data mining},
  pages={157--166},
  year={2005}
}

@article{pevny2016loda,
  title={Loda: Lightweight on-line detector of anomalies},
  author={Pevn{\`y}, Tom{\'a}{\v{s}}},
  journal={Machine Learning},
  volume={102},
  number={2},
  pages={275--304},
  year={2016},
  publisher={Springer}
}

@article{goldstein2012histogram,
  title={Histogram-based outlier score (hbos): A fast unsupervised anomaly detection algorithm},
  author={Goldstein, Markus and Dengel, Andreas},
  journal={KI-2012: poster and demo track},
  volume={1},
  pages={59--63},
  year={2012},
  publisher={Saarbruecken, Germany}
}

@inproceedings{li2020copod,
  title={Copod: copula-based outlier detection},
  author={Li, Zheng and Zhao, Yue and Botta, Nicola and Ionescu, Cezar and Hu, Xiyang},
  booktitle={2020 IEEE international conference on data mining (ICDM)},
  pages={1118--1123},
  year={2020},
  organization={IEEE}
}

@article{li2022ecod,
  title={Ecod: Unsupervised outlier detection using empirical cumulative distribution functions},
  author={Li, Zheng and Zhao, Yue and Hu, Xiyang and Botta, Nicola and Ionescu, Cezar and Chen, George H},
  journal={IEEE Transactions on Knowledge and Data Engineering},
  volume={35},
  number={12},
  pages={12181--12193},
  year={2022},
  publisher={IEEE}
}

@article{durani2026anomaly,
  title={Anomaly Detection by an Ensemble of Random Pairs of Hyperspheres},
  author={Durani, Walid and Leiber, Collin and Durani, Khalid and Plant, Claudia and B{\"o}hm, Christian},
  journal={Advances in Neural Information Processing Systems},
  volume={38},
  pages={166168--166207},
  year={2026}
}

@article{shyu2003novel,
  title={A novel anomaly detection scheme based on principal component classifier},
  author={Shyu, Mei-Ling and Chen, Shu-Ching and Sarinnapakorn, Kanoksri and Chang, LiWu},
  year={2003}
}

@inproceedings{kriegel2009outlier,
  title={Outlier detection in axis-parallel subspaces of high dimensional data},
  author={Kriegel, Hans-Peter and Kr{\"o}ger, Peer and Schubert, Erich and Zimek, Arthur},
  booktitle={Pacific-asia conference on knowledge discovery and data mining},
  pages={831--838},
  year={2009},
  organization={Springer}
}

@article{zimek2014ensembles,
  title={Ensembles for unsupervised outlier detection: challenges and research questions a position paper},
  author={Zimek, Arthur and Campello, Ricardo JGB and Sander, J{\"o}rg},
  journal={Acm Sigkdd Explorations Newsletter},
  volume={15},
  number={1},
  pages={11--22},
  year={2014},
  publisher={ACM New York, NY, USA}
}

@inproceedings{zong2018deep,
  title={Deep autoencoding gaussian mixture model for unsupervised anomaly detection},
  author={Zong, Bo and Song, Qi and Min, Martin Renqiang and Cheng, Wei and Lumezanu, Cristian and Cho, Daeki and Chen, Haifeng},
  booktitle={International conference on learning representations},
  year={2018}
}

@article{bergman2020classification,
  title={Classification-based anomaly detection for general data},
  author={Bergman, Liron and Hoshen, Yedid},
  journal={arXiv preprint arXiv:2005.02359},
  year={2020}
}

@inproceedings{qiu2021neural,
  title={Neural transformation learning for deep anomaly detection beyond images},
  author={Qiu, Chen and Pfrommer, Timo and Kloft, Marius and Mandt, Stephan and Rudolph, Maja},
  booktitle={International conference on machine learning},
  pages={8703--8714},
  year={2021},
  organization={PMLR}
}

@inproceedings{shenkar2022anomaly,
  title={Anomaly detection for tabular data with internal contrastive learning},
  author={Shenkar, Tom and Wolf, Lior},
  booktitle={International conference on learning representations},
  year={2022}
}

@inproceedings{xu2023fascinating,
  title={Fascinating supervisory signals and where to find them: Deep anomaly detection with scale learning},
  author={Xu, Hongzuo and Wang, Yijie and Wei, Juhui and Jian, Songlei and Li, Yizhou and Liu, Ning},
  booktitle={International Conference on Machine Learning},
  pages={38655--38673},
  year={2023},
  organization={PMLR}
}

@article{wang2019unsupervised,
  title={Unsupervised representation learning by predicting random distances},
  author={Wang, Hu and Pang, Guansong and Shen, Chunhua and Ma, Congbo},
  journal={arXiv preprint arXiv:1912.12186},
  year={2019}
}

@inproceedings{liu2021rca,
  title={Rca: A deep collaborative autoencoder approach for anomaly detection},
  author={Liu, Boyang and Wang, Ding and Lin, Kaixiang and Tan, Pang-Ning and Zhou, Jiayu},
  booktitle={IJCAI: proceedings of the conference},
  volume={2021},
  pages={1505},
  year={2021}
}

@inproceedings{ramaswamy2000efficient,
  title={Efficient algorithms for mining outliers from large data sets},
  author={Ramaswamy, Sridhar and Rastogi, Rajeev and Shim, Kyuseok},
  booktitle={Proceedings of the 2000 ACM SIGMOD international conference on Management of data},
  pages={427--438},
  year={2000}
}

@inproceedings{pang2018learning,
  title={Learning representations of ultrahigh-dimensional data for random distance-based outlier detection},
  author={Pang, Guansong and Cao, Longbing and Chen, Ling and Liu, Huan},
  booktitle={Proceedings of the 24th ACM SIGKDD international conference on knowledge discovery \& data mining},
  pages={2041--2050},
  year={2018}
}

@article{dasgupta2003elementary,
  title={An elementary proof of a theorem of Johnson and Lindenstrauss},
  author={Dasgupta, Sanjoy and Gupta, Anupam},
  journal={Random Structures \& Algorithms},
  volume={22},
  number={1},
  pages={60--65},
  year={2003},
  publisher={Wiley Online Library}
}

@article{johnson1984extensions,
  title={Extensions of Lipschitz mappings into a Hilbert space},
  author={Johnson, William B and Lindenstrauss, Joram and others},
  journal={Contemporary mathematics},
  volume={26},
  number={189-206},
  pages={1},
  year={1984}
}

@article{rosenblatt1956remarks,
  title={Remarks on Some Nonparametric Estimates of a Density Function},
  author={Rosenblatt, Murray},
  journal={The Annals of Mathematical Statistics},
  pages={832--837},
  year={1956},
  publisher={JSTOR}
}

@article{parzen1962estimation,
  title={On estimation of a probability density function and mode},
  author={Parzen, Emanuel},
  journal={The annals of mathematical statistics},
  volume={33},
  number={3},
  pages={1065--1076},
  year={1962},
  publisher={JSTOR}
}

@article{rahimi2007random,
  title={Random features for large-scale kernel machines},
  author={Rahimi, Ali and Recht, Benjamin},
  journal={Advances in neural information processing systems},
  volume={20},
  year={2007}
}

@article{muandet2016kernel,
  title={Kernel Mean Embedding of Distributions: A Review and Beyond},
  author={Muandet, Krikamol and Fukumizu, Kenji and Sriperumbudur, Bharath and Sch{\"o}lkopf, Bernhard},
  journal={arXiv preprint arXiv:1605.09522},
  year={2016}
}

@inproceedings{yang2014quasi,
  title={Quasi-Monte Carlo feature maps for shift-invariant kernels},
  author={Yang, Jiyan and Sindhwani, Vikas and Avron, Haim and Mahoney, Michael},
  booktitle={International Conference on Machine Learning},
  pages={485--493},
  year={2014},
  organization={PMLR}
}

@article{hoeffding1963probability,
  title={Probability inequalities for sums of bounded random variables},
  author={Hoeffding, Wassily},
  journal={Journal of the American statistical association},
  volume={58},
  number={301},
  pages={13--30},
  year={1963},
  publisher={Taylor \& Francis}
}

@article{serfling1974probability,
  title={Probability inequalities for the sum in sampling without replacement},
  author={Serfling, Robert J},
  journal={The Annals of Statistics},
  pages={39--48},
  year={1974},
  publisher={JSTOR}
}

@article{sutherland2015error,
  title={On the error of random Fourier features},
  author={Sutherland, Danica J and Schneider, Jeff},
  journal={arXiv preprint arXiv:1506.02785},
  year={2015}
}

@book{silverman2018density,
  title={Density estimation for statistics and data analysis},
  author={Silverman, Bernard W},
  year={2018},
  publisher={Routledge}
}

@inproceedings{aggarwal2001outlier,
  title={Outlier detection for high dimensional data},
  author={Aggarwal, Charu C and Yu, Philip S},
  booktitle={Proceedings of the 2001 ACM SIGMOD international conference on Management of data},
  pages={37--46},
  year={2001}
}

@article{han2022adbench,
  title={Adbench: Anomaly detection benchmark},
  author={Han, Songqiao and Hu, Xiyang and Huang, Hailiang and Jiang, Minqi and Zhao, Yue},
  journal={Advances in neural information processing systems},
  volume={35},
  pages={32142--32159},
  year={2022}
}

@article{jiang2023adgym,
  title={Adgym: Design choices for deep anomaly detection},
  author={Jiang, Minqi and Hou, Chaochuan and Zheng, Ao and Han, Songqiao and Huang, Hailiang and Wen, Qingsong and Hu, Xiyang and Zhao, Yue},
  journal={Advances in Neural Information Processing Systems},
  volume={36},
  pages={70179--70207},
  year={2023}
}

@inproceedings{chen2025pyod,
  title={Pyod 2: A python library for outlier detection with llm-powered model selection},
  author={Chen, Sihan and Qian, Zhuangzhuang and Siu, Wingchun and Hu, Xingcan and Li, Jiaqi and Li, Shawn and Qin, Yuehan and Yang, Tiankai and Xiao, Zhuo and Ye, Wanghao and others},
  booktitle={Companion Proceedings of the ACM on Web Conference 2025},
  pages={2807--2810},
  year={2025}
}

@article{chandola2009anomaly,
  title={Anomaly detection: A survey},
  author={Chandola, Varun and Banerjee, Arindam and Kumar, Vipin},
  journal={ACM computing surveys (CSUR)},
  volume={41},
  number={3},
  pages={1--58},
  year={2009},
  publisher={ACM New York, NY, USA}
}

@article{ruff2021unifying,
  title={A unifying review of deep and shallow anomaly detection},
  author={Ruff, Lukas and Kauffmann, Jacob R and Vandermeulen, Robert A and Montavon, Gr{\'e}goire and Samek, Wojciech and Kloft, Marius and Dietterich, Thomas G and M{\"u}ller, Klaus-Robert},
  journal={Proceedings of the IEEE},
  volume={109},
  number={5},
  pages={756--795},
  year={2021},
  publisher={IEEE}
}

@inproceedings{ruff2018deep,
  title={Deep one-class classification},
  author={Ruff, Lukas and Vandermeulen, Robert and Goernitz, Nico and Deecke, Lucas and Siddiqui, Shoaib Ahmed and Binder, Alexander and M{\"u}ller, Emmanuel and Kloft, Marius},
  booktitle={International conference on machine learning},
  pages={4393--4402},
  year={2018},
  organization={PMLR}
}

@article{pang2021deep,
  title={Deep learning for anomaly detection: A review},
  author={Pang, Guansong and Shen, Chunhua and Cao, Longbing and Hengel, Anton Van Den},
  journal={ACM computing surveys (CSUR)},
  volume={54},
  number={2},
  pages={1--38},
  year={2021},
  publisher={ACM New York, NY, USA}
}

@ARTICLE{xu2023deep,
   author={Xu, Hongzuo and Pang, Guansong and Wang, Yijie and Wang, Yongjun},
   journal={IEEE Transactions on Knowledge and Data Engineering},
   title={Deep Isolation Forest for Anomaly Detection},
   year={2023},
   volume={35},
   number={12},
   pages={12591--12604},
   doi={10.1109/TKDE.2023.3270293}
}

@ARTICLE{xu2024calibrated,
   title={Calibrated one-class classification for unsupervised time series anomaly detection},
   author={Xu, Hongzuo and Wang, Yijie and Jian, Songlei and Liao, Qing and Wang, Yongjun and Pang, Guansong},
   journal={IEEE Transactions on Knowledge and Data Engineering},
   year={2024},
   publisher={IEEE}
}

@article{chatterjee2022iot,
  title={IoT anomaly detection methods and applications: A survey},
  author={Chatterjee, Ayan and Ahmed, Bestoun S},
  journal={Internet of Things},
  volume={19},
  pages={100568},
  year={2022},
  publisher={Elsevier}
}
\bibliographystyle{iclr2027_conference}

\newpage
\appendix

The appendix here provides additional details for the submission titled: “RGLD: Randomized Global-Local Density Estimation for Tabular Anomaly Detection”. The appendix is organized as follows:

{\setstretch{2}
\begin{enumerate}[label=\Alph*.]
    \large
    \item \hyperref[app:notation]{\textbf{List of Notation}}
    \item \hyperref[app:discussion]{\textbf{Discussion}}
    \item \hyperref[app:limitations]{\textbf{Limitations}}
    \item \hyperref[app:pseudocode]{\textbf{Full RGLD Pseudocode}}
    \item \hyperref[app:hyperparameter]{\textbf{Hyperparameters for Two RGLD Variants}}
    \item \hyperref[app:prop1]{\textbf{Proof of Proposition~\ref{prop:feature-bagging-contrast}}}
    \item \hyperref[app:prop1]{\textbf{Proof of Proposition~\ref{prop:rf_kde}}}
    \item \hyperref[app:prop2]{\textbf{Proof of Proposition~\ref{prop:kNN_reference}}}
    \item \hyperref[app:scaling]{\textbf{Scaling in Dataset Size and Dataset Dimension}}
    \item \hyperref[app:aderh_comparison]{\textbf{Detailed Comparison with ADERH}}
    \item \hyperref[app:frontier]{\textbf{Full Pareto Frontier Graph}}
    \item \hyperref[app:dataset]{\textbf{Dataset Detail}}
    \item \hyperref[app:benchmark]{\textbf{Full Benchmark Results}}
    \item \hyperref[app:code]{\textbf{Code Availability}}
\end{enumerate}
}
\newpage

\section{List of Notation}
\label{app:notation}

We hereby provide a list of notations used in this paper:

\begin{table*}[h]
\caption{List of notations.}
\begin{center}
\small
\setlength{\tabcolsep}{6pt}
\begin{tabular}{p{10em} || p{32em}} 
 \toprule
 \toprule
 Symbol & Meaning \\
 \midrule
 \centering \(X=\{x_i\}_{i=1}^{n}\) & Unlabeled dataset with \(n\) samples. \\
 \centering \(d\) & Ambient data dimension. \\
 \centering \(s(x)\) & Final RGLD anomaly score; larger values indicate more anomalous samples. \\
 \centering \(S_t\) & Feature subset sampled for randomized view \(t\). \\
 \centering $r$ & Projected dimension after random projection. \\
 \centering \(R_t\) & Random projection matrix for view \(t\). \\
 \centering \(z_t(x)=x_{S_t}R_t\) & Projected representation of \(x\) in randomized view \(t\). \\
 \centering \(T_g, T_\ell\) & Number of global density views and projected-neighbor views. \\
 \centering \(\sigma_{t,m}\) & Bandwidth for global view \(t\) and scale \(m\). \\
 \centering \(D\) & Random-feature dimension. \\
 \centering \(\phi_{t,\sigma}(z)\) & Random-feature map for projected sample \(z\). \\
 \centering \(\mu_{t,\sigma}\) & Empirical random-feature mean embedding. \\
 \centering \(\widehat{p}_{t,\sigma}(x)\), \(\widetilde{p}_{t,\sigma}(x)\) & Exact projected KDE score and its random-feature approximation. \\
 \centering \(s_{\mathrm{glob}}(x)\) & Rank-aggregated global density branch score. \\
 \centering \(C_t\) & Sampled reference set for projected-neighbor view \(t\). \\
 \centering \(k\) & Number of nearest reference points. \\
 \centering \(a_t^{\mathrm{loc}}(x)\) & Projected-neighbor local isolation score. \\
 \centering \(s_{\mathrm{loc}}(x)\) & Rank-aggregated local neighbor branch score. \\
 \centering \(\operatorname{RankAgg}(\cdot)\) & Rank-based aggregation operator. \\
 \centering \(B_t(x,\rho)\), \(q_t(x,\rho)\) & Projected radius-\(\rho\) neighborhood of \(x\), and its empirical mass. \\
 \bottomrule
 \bottomrule
\end{tabular}
\label{notation}
\end{center}
\end{table*}

\section{Discussion}
\label{app:discussion}

\textbf{Practical Guide for Hyperparameter Selection: }RGLD's hyperparameters mainly control the tradeoff between approximation fidelity, ensemble diversity, and runtime. The random-feature dimension $D$ controls the global density approximation: by Proposition~\ref{prop:rf_kde}, the random-feature KDE error decreases as $O(D^{-1/2})$, so larger $D$ improves approximation quality with diminishing returns. In our experiments, a modest value such as 10 was sufficient to provide a strong accuracy-efficiency tradeoff. The reference set size $|C_t|$ controls the local-support approximation: by Proposition~\ref{prop:kNN_reference}, the probability of missing a radius-$r$ neighborhood is at most $\exp(-|C_t|q_t(x,\rho))$, so larger $|C_t|$ improves local coverage exponentially while increasing neighbor-scoring cost linearly. Thus, within a computational budget, larger $D$ and $|C_t|$ generally improve the fidelity of the global and local branch approximations. The number of global and local views, $T_g$ and $T_\ell$, increases ensemble diversity and the chance of observing informative views, with runtime growing approximately linearly in the number of views. The bandwidth multipliers $\alpha_m$ define the density scales used by the global branch; a small grid such as $\{0.25,0.5,1.0,2.0\}$ covers both fine and coarse anomaly scales without high cost. Finally, the number of neighbors $k$ controls the locality of the projected-neighbor score, smaller $k$ emphasizes fine local isolation, while a larger $k$ gives smoother but less local support estimates.

\textbf{The Need for Random Projection: }Random projection is used in RGLD to make both scoring branches cheaper while retaining useful geometric information. By the Johnson-Lindenstrauss principle, random projections can approximately preserve pairwise distances in a lower-dimensional space~\citep{johnson1984extensions}, which is important because both density and neighbor-based anomaly scores depend on this geometry. In the global branch, projecting the data reduces the cost of random-feature density estimation and helps avoid directly applying KDE-style scoring in the original high-dimensional space, where density estimates are sensitive to dimension and bandwidth choice~\citep{silverman2018density}. In the local branch, projection reduces the cost of neighbor-distance computation by replacing full-dimensional distances with low-dimensional projected distances. RGLD therefore uses random projection as a lightweight mechanism for producing inexpensive density and neighborhood estimates across multiple randomized views.

\section{Limitations}
\label{app:limitations}

\textbf{Potential Failure Mode: }RGLD relies on the assumption that useful anomaly evidence can be exposed through randomized views of the observed tabular feature space. This makes the method efficient and fully unsupervised, but it may be less effective when anomaly signals are extremely weak, highly diffuse across many features, or only visible through complex nonlinear representations. Feature bagging is most helpful when some subsets of features provide stronger anomaly contrast than the full space; when the raw features do not contain meaningful density or neighborhood structure, representation-learning methods or domain-specific features may be more suitable.

\textbf{Anomaly Signals: }RGLD focuses on scalable approximations to density-based and neighborhood-based anomaly signals. This covers a broad and effective class of unsupervised detectors, but abnormality can also appear 
through other mechanisms. For example, anomalies may be characterized by reconstruction error, violation of feature dependencies, abnormal marginal tail behavior, or inconsistency with learned semantic representations. RGLD's randomized ensemble framework is compatible with these signals in principle, but this paper focuses on global density and local neighbor support. Future work could extend RGLD by incorporating additional scalable score approximations into the same randomized rank-aggregation framework.

\textbf{Extension to Broader Data Modalities or Settings: }Our experiments focus on static tabular anomaly detection. Extending RGLD to time series, graphs, images, streaming data, or semi-supervised settings would require additional design choices, such as temporal modeling, graph-aware decisions, and online updates. We view these directions as complementary to the present work and as promising paths toward more broadly applicable, scalable anomaly detection.

\section{RGLD Pseudocode}
\label{app:pseudocode}

Algorithm~\ref{alg:RGLD} summarizes RGLD. The global branch fits randomized density estimators using random-feature mean embeddings, while the local neighbor branch fits sampled-reference local isolation estimators. Scores are rank-aggregated within each branch and then across branches.

\begin{algorithm}[H]
\caption{\textsc{RGLD}}
\label{alg:RGLD}
\KwIn{Unlabeled data \(X=\{x_i\}_{i=1}^{n}\); global views \(T_g\); local views \(T_\ell\); bandwidth multipliers \(\{\alpha_m\}_{m=1}^{M}\); reference size \(c\); neighbors \(k\); Random feature dimension $D$}
\KwOut{Anomaly scores \(s(x_i)\)}

Preprocess \(X\) with dataset statistics;

\BlankLine
\textbf{Global density branch}\;
\For{\(t=1,\ldots,T_g\)}{
    Sample feature subset and random projection; obtain projected data \(Z_t\)\;
    Estimate base bandwidth \(\widehat{\sigma}_t\) from sampled pairwise distances in \(Z_t\)\;
    Construct multi-scale bandwidths \(\sigma_{t,m}=\alpha_m\widehat{\sigma}_t\)\;
    Fit random-feature mean embeddings \(\{\mu_{t,m}\}_{m=1}^{M}\)\;
    Score samples by negative random-feature density at each scale\;
    Standardize and average multi-scale scores to obtain \(a_t^{\mathrm{glob}}\)\;
}
Aggregate global scores: \(s_{\mathrm{glob}}=\operatorname{RankAgg}(a_1^{\mathrm{glob}},\ldots,a_{T_g}^{\mathrm{glob}})\)\;

\BlankLine
\textbf{local neighbor branch}\;
\For{\(t=1,\ldots,T_\ell\)}{
    Sample feature subset and random projection; obtain projected data \(Z_t^\ell\)\;
    Sample reference set \(C_t\subset Z_t^\ell\) with \(|C_t|=c\)\;
    Score each sample by its average distance to the \(k\) nearest points in \(C_t\)\;
    Store the resulting local score vector \(a_t^{\mathrm{loc}}\)\;
}
Aggregate local scores: \(s_{\mathrm{loc}}=\operatorname{RankAgg}(a_1^{\mathrm{loc}},\ldots,a_{T_\ell}^{\mathrm{loc}})\)\;

\BlankLine
\textbf{Final score}\;
Return \(s=\operatorname{RankAgg}(s_{\mathrm{glob}},s_{\mathrm{loc}})\)\;
\end{algorithm}

\textbf{Label-free preprocessing selection: }RGLD selects preprocessing using only training-set statistics. It chooses among standard scaling, robust scaling, and quantile transformation based on scale heterogeneity, tail heaviness, skewness, sparsity, and discreteness. Robust scaling is used as the safe default when feature geometry may be distorted by outliers or scale mismatch; quantile transformation is used for strongly skewed continuous data; otherwise, standard scaling is used.

\paragraph{Automatic projection dimension: }RGLD uses a fixed rule for random projection dimension based only on input dimension \(d\): \(r=2\) for \(d\leq3\), \(r=5\) for \(d\leq20\), \(r=8\) for \(d\leq40\), \(r=12\) for \(d\leq100\), and \(r=16\) otherwise. The local neighbor branch uses the same rule.

\section{Hyperparameters for Two RGLD Variants}
\label{app:hyperparameter}

We use two RGLD variants in the experiments: \textbf{RGLD} and \textbf{RGLD-Fast}. RGLD is the default configuration used for the main benchmark, while RGLD-Fast reduces the number of randomized views, random-feature dimension, reference size to improve runtime. The hyperparameters used for each variant are shown in Table~\ref{tab:RGLD_hyperparameters}.
\begin{table}[h]
\caption{
\textbf{Hyperparameters for RGLD and RGLD-Fast.}
RGLD-Fast keeps the same design as RGLD but reduces the ensemble and random-feature budgets for lower runtime.
}
\centering
\setlength{\tabcolsep}{5pt}
\begin{tabular}{l c c}
\toprule
Hyperparameter & RGLD & RGLD-Fast \\
\midrule
Random-feature dimension $D$ & 10 & 5 \\
Bandwidth multipliers $\alpha_m$ & $\{0.25,0.5,1.0,2.0\}$ & $\{0.25,0.5,1.0,2.0\}$ \\
Global randomized views $T_g$ & 40 & 20 \\
Projected-neighbor views $T_\ell$ & 10 & 5 \\
Reference set size $|C_t|$ & 400 & 200 \\
Number of neighbors $k$ & 10 & 10 \\
\bottomrule
\end{tabular}
\label{tab:RGLD_hyperparameters}
\end{table}

For both variants, each randomized view samples approximately $30\%$ of the input features, subject to a minimum of $3$ selected features, before applying a random projection. For each global randomized view, the base bandwidth is estimated as the median Euclidean distance over 8192 sampled pairs in the projected space, and the final set of bandwidths is obtained by multiplying this view-specific base bandwidth by the fixed scale multipliers.

\section{Proof of Proposition~\ref{prop:feature-bagging-contrast}}
\label{app:prop1}

Let $S^\star$ be the informative feature set and $S$ be a feature subset with $|S|=m$.
By definition,
$$
\pi(S)=\frac{|S\cap S^\star|}{m}
$$
The expected squared anomaly-normal separation in $S$ decomposes into informative and nuisance
coordinates:
$$
\mathbb{E}\|(x_a-x_n)_S\|_2^2
=
|S\cap S^\star|\mu_{\mathrm{sig}}
+
|S\setminus S^\star|\mu_{\mathrm{noise}}
$$
Dividing by $m$ gives
$$
C(S)
=
\frac{|S\cap S^\star|}{m}\mu_{\mathrm{sig}}
+
\left(1-\frac{|S\cap S^\star|}{m}\right)\mu_{\mathrm{noise}}
=
\pi(S)\mu_{\mathrm{sig}}+(1-\pi(S))\mu_{\mathrm{noise}}.
$$
For the full feature space $[d]$, the informative-feature fraction is $\pi([d])=s/d$, so
$$
C([d])
=
\frac{s}{d}\mu_{\mathrm{sig}}
+
\left(1-\frac{s}{d}\right)\mu_{\mathrm{noise}}
$$
Therefore,
$$
C(S)-C([d])
=
\left(\pi(S)-\frac{s}{d}\right)
(\mu_{\mathrm{sig}}-\mu_{\mathrm{noise}})
$$
Since $\mu_{\mathrm{sig}}>\mu_{\mathrm{noise}}$, if $\pi(S)>s/d$, then
$$
C(S)-C([d])>0
$$
and hence $C(S)>C([d])$.

Finally, since $\pi(S)=|S\cap S^\star|/m$, we have
$$
\pi(S)>\frac{s}{d}
\quad \Longleftrightarrow \quad
\frac{|S\cap S^\star|}{m}>\frac{s}{d}
\quad \Longleftrightarrow \quad
|S\cap S^\star|>\frac{ms}{d}
$$
Thus,
$$
\Pr\left(\pi(S)>\frac{s}{d}\right)
=
\Pr\left(|S\cap S^\star|>\frac{ms}{d}\right)
$$

\begin{flushright} $\blacksquare$ \end{flushright}

When the feature subset $S$ is sampled uniformly without replacement from $[d]$, the random variable $J=|S\cap S^\star|$ follows a hypergeometric distribution. Therefore, the probability that a randomized feature subset is enriched relative to the full feature space is:
$$
p_{\mathrm{enrich}} = \Pr\left(\pi(S)>\frac{s}{d}\right)
=
\Pr\left(J>\frac{ms}{d}\right)
=
\sum_{j>\frac{ms}{d}}
\frac{\binom{s}{j}\binom{d-s}{m-j}}{\binom{d}{m}}
$$
This quantity characterizes how often randomized views are expected to have higher informative feature concentration than the full space. Across $T$ independently sampled views, the expected number of enriched views is $T p_{\mathrm{enrich}}$, and the observed fraction concentrates around
$p_{\mathrm{enrich}}$ by standard concentration inequalities. Thus, when $p_{\mathrm{enrich}}$ is
non-negligible, RGLD can include a meaningful fraction of feature-bagged detectors with stronger
subspace contrast, allowing these views to contribute consistently to the rank-aggregated anomaly
score.

\section{Proof of Proposition~\ref{prop:rf_kde}}
\label{app:prop2}

We prove the result for a fixed randomized view $t$ and bandwidth $\sigma$. Conditioning on the sampled feature subset and projection, the projected samples $z_i^{(t)}$ are fixed. Thus, the only randomness comes from the random features used to approximate the Gaussian kernel.

For simplicity, write $z=z_t(x)$ and $z_i=z_i^{(t)}$. The exact projected KDE score is
$$
    \widehat{p}_{t,\sigma}(x)
    =
    \frac{1}{N}\sum_{i=1}^{N} k_{\sigma}(z,z_i)
$$
Using cosine random features, the RGLD approximation can be written as
$$
    \widetilde{p}_{t,\sigma}(x)
    =
    \frac{1}{D}
    \sum_{\ell=1}^{D}
    h_\ell(x)
$$
where
$$
    h_\ell(x)
    =
    \frac{2}{N}
    \sum_{i=1}^{N}
    \cos\left(\frac{\omega_\ell^\top z}{\sigma}+b_\ell\right)
    \cos\left(\frac{\omega_\ell^\top z_i}{\sigma}+b_\ell\right)
$$
By the standard random Fourier feature theory, each feature gives an unbiased estimate of the Gaussian kernel, so
$$
    \mathbb{E}[h_\ell(x)]
    =
    \frac{1}{N}
    \sum_{i=1}^{N}
    k_{\sigma}(z,z_i)
    =
    \widehat{p}_{t,\sigma}(x)
$$
Therefore, $\widetilde{p}_{t,\sigma}(x)$ is an unbiased estimator of $\widehat{p}_{t,\sigma}(x)$.

It remains to bound its deviation. Since each cosine term lies in $[-1,1]$, we have $h_\ell(x)\in[-2,2]$. The variables $\{h_\ell(x)\}_{\ell=1}^{D}$ are independent across random features. Applying Hoeffding's inequality to the average $\frac{1}{D}\sum_{\ell=1}^{D} h_\ell(x)$, for any $\epsilon>0$,
$$
    \Pr\left(
    \left|
    \widetilde{p}_{t,\sigma}(x)
    -
    \widehat{p}_{t,\sigma}(x)
    \right|
    \ge \epsilon
    \right)
    \le
    2\exp\left(-\frac{D\epsilon^2}{8}\right)
$$
Setting the right-hand side equal to $\delta$ gives, with probability at least $1-\delta$,
$$
    \left|
    \widetilde{p}_{t,\sigma}(x)
    -
    \widehat{p}_{t,\sigma}(x)
    \right|
    \le
    \sqrt{\frac{8\log(2/\delta)}{D}}
    =
    O\left(
    \sqrt{\frac{\log(1/\delta)}{D}}
    \right)
$$

\begin{flushright} $\blacksquare$ \end{flushright}

This proves the stated bound. The proposition is stated for original cosine random Fourier features; in our implementation, QMC features are used as a variance-reduced alternative, while the bound serves as the standard RFF approximation upper bound.

\section{Proof of Proposition~\ref{prop:kNN_reference}}
\label{app:prop3}

Fix a projected view $t$, an evaluation point $x$, and radius $\rho>0$. Let 
$B_t(x,\rho)$ be the set of projected training samples within distance $\rho$ of $z_t(x)$, and let $b=|B_t(x,\rho)|$. By definition, $q_t(x,\rho)=b/n$. Let $m=|C_t|$ be the sampled reference-set size.

Since $C_t$ is sampled uniformly without replacement from the $n$ projected training samples, the probability that it contains no point from $B_t(x,\rho)$ is
$$
    \Pr\left[
    C_t \cap B_t(x,\rho)=\emptyset
    \right]
    =
    \frac{\binom{n-b}{m}}{\binom{n}{m}}
$$
If $m>n-b$, the probability is $0$, and the desired bound holds trivially. Assuming $m\le n-b$, expanding the ratio gives
$$
    \frac{\binom{n-b}{m}}{\binom{n}{m}}
    =
    \prod_{j=0}^{m-1}
    \frac{n-b-j}{n-j}
$$
For each $j\ge 0$, we have
$$
    \frac{n-b-j}{n-j}
    \le
    \frac{n-b}{n}
    =
    1-\frac{b}{n}
    =
    1-q_t(x,\rho)
$$
Therefore,
$$
    \Pr\left[
    C_t \cap B_t(x,\rho)=\emptyset
    \right]
    \le
    \left(1-q_t(x,\rho)\right)^m
$$
Using $1-u\le e^{-u}$ for $u\in[0,1]$, we obtain
$$
    \Pr\left[
    C_t \cap B_t(x,\rho)=\emptyset
    \right]
    \le
    \exp\left(-m q_t(x,\rho)\right)
    =
    \exp\left(-|C_t|q_t(x,\rho)\right)
$$
Thus, with probability at least $1-\exp(-|C_t|q_t(x,\rho))$, the reference set contains at least one projected training sample within distance $\rho$ of $z_t(x)$. This proves the proposition.

\begin{flushright} $\blacksquare$ \end{flushright}

Proposition~\ref{prop:kNN_reference} characterizes the sampling gap introduced by replacing full-reference $k$NN with sampled-reference $k$NN. In the full $k$NN score, a point $x$ can use all training samples to find nearby support. In RGLD, $x$ only searches within the sampled reference set $C_t$. The bound shows that if the projected neighborhood $B_t(x,\rho)$ contains non-negligible empirical mass $q_t(x,\rho)$, then $C_t$ is unlikely to miss this neighborhood; the miss probability decays exponentially in $|C_t|q_t(x,\rho)$. Thus, when $x$ is well supported in a projected view, the sampled-reference score is likely to find nearby points and remain small. Conversely, if the sampled-reference score is large, this suggests either that $x$ has genuinely weak local support in that view or that the reference sample missed a very small neighborhood. RGLD mitigates the latter failure mode by aggregating multiple independently sampled projected-neighbor views and reference sets. Therefore, the proposition justifies sampled-reference $k$NN as a scalable approximation to local support.

\section{Scaling in Dataset Size and Dataset Dimension}
\label{app:scaling}

\begin{figure}[H]
    \centering
    \includegraphics[width=1\linewidth]{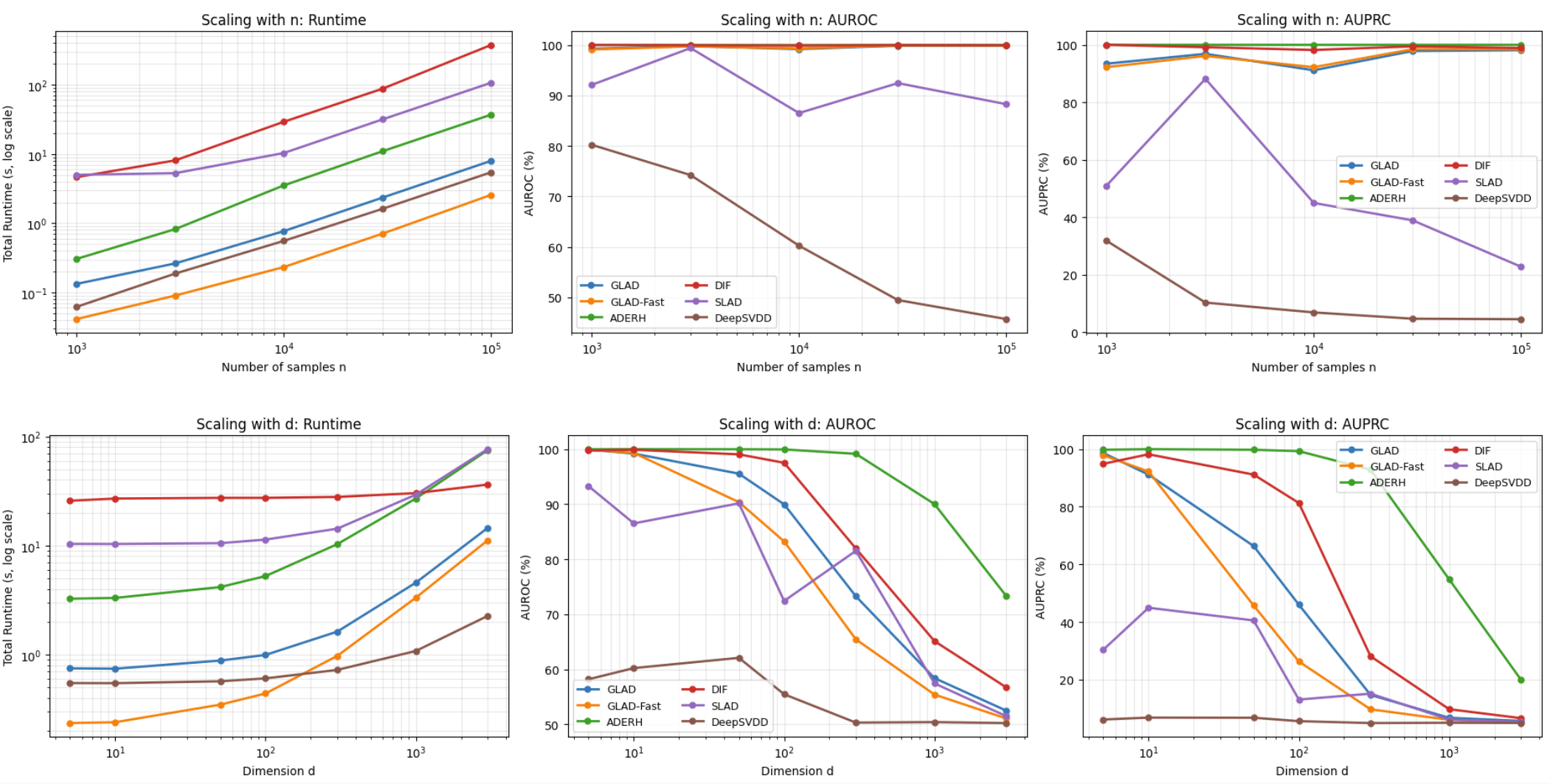}
    \caption{
        \textbf{Scaling with dataset size and ambient dimension.}
        Top: runtime, AUROC, and AUPRC as the number of samples $n$ increases. Bottom: runtime, AUROC, and AUPRC as the ambient dimension $d$ increases. RGLD and RGLD-Fast maintain strong performance and smooth runtime growth with $n$. As $d$ becomes very large, accuracy decreases due to dilution of sparse anomaly signals by nuisance dimensions.
        }
        \label{fig:scaling_nd}
\end{figure}

We further evaluate how RGLD scales with the number of samples $n$ and the ambient dimension $d$ using controlled synthetic data. We use controlled synthetic data to isolate scaling behavior. Normal samples are drawn from a Gaussian mixture, while anomalies follow a local-shift mechanism where they are sampled from the same cluster structure and then shifted along a small randomly chosen subset of informative dimensions. We vary either the number of samples $n$ or the ambient dimension $d$, keeping the anomaly ratio and shift magnitude fixed. In the dimension-scaling experiment, the anomaly signal remains confined to a small number of informative features as $d$ grows, making detection increasingly challenging. The goal of this experiment is to complement the main benchmark by isolating runtime and accuracy trends under increasing dataset size and dimensionality.

Figure~\ref{fig:scaling_nd} shows the results. When scaling the number of samples $n$, RGLD and RGLD-Fast maintain high AUROC and AUPRC across the full range while runtime grows smoothly. This behavior is consistent with the complexity analysis in Section~\ref{sec:analysis}: RGLD avoids explicit pairwise KDE and full-reference $k$NN computations, replacing them with randomized density and sampled-reference local-support estimates. RGLD-Fast further reduces runtime while preserving similar accuracy, illustrating the controllable accuracy--efficiency tradeoff of the randomized-view design.

When scaling the ambient dimension $d$, runtime increases for all methods, but RGLD remains efficient over most of the range. Accuracy decreases at very large $d$, especially in AUPRC, reflecting the increasing difficulty of detecting sparse anomaly signals as they are diluted by many nuisance dimensions. This stress test highlights both the benefit and limitation of randomized feature-bagged views: they improve efficiency and can expose subspace anomalies, but extremely high ambient dimension remains challenging when the informative anomaly signal occupies only a small fraction of features.

Interestingly, several deep methods (DIF, DeepSVDD) exhibit flatter runtime growth than statistical methods as the ambient dimension $d$ increases. This likely reflects their use of fixed neural architectures, whose runtime is less sensitive to large input dimension. By contrast, statistical detectors often operate directly on distances, projections, or neighborhoods in the input space. Therefore, RGLD's main scaling advantage is most evident in sample size $n$, while very high-dimensional inputs suggest a direction for further optimization through sparser or more adaptive projections.

\section{Detailed Comparison with ADERH}
\label{app:aderh_comparison}

ADERH is the strongest statistical baseline in our main benchmark, achieving the highest mean AUROC and AUPRC among the evaluated methods. Since both ADERH and RGLD are randomized statistical detectors, we provide a more detailed comparison between the two methods. Table~\ref{tab:rgld_aderh_summary} summarizes their aggregate performance over the 47 ADBench datasets. ADERH obtains slightly higher mean AUROC and AUPRC. In contrast, RGLD achieves more first-place wins among all evaluated methods in both AUROC and AUPRC, while being substantially faster. In particular, RGLD runs in 0.22 seconds on average, compared with 1.06 seconds for ADERH, corresponding to an average runtime advantage of approximately 4.86x.

\begin{table}[h]
\caption{Head-to-head comparison between RGLD and ADERH over 47 ADBench datasets.}
\centering
\begin{tabular}{lcc}
\toprule
Metric & RGLD & ADERH \\
\midrule
Mean AUROC & 77.50 & 78.18 \\
Mean AUPRC & 38.59 & 39.72 \\
Average rank & 9.72 & 8.22 \\
AUROC wins among all methods & 8 & 4 \\
AUPRC wins among all methods & 5 & 3 \\
Mean runtime (s) & 0.22 & 1.06 \\
Relative runtime & 1.00x & 4.86x slower \\
\bottomrule
\end{tabular}
\label{tab:rgld_aderh_summary}
\end{table}

To assess whether the observed differences are statistically meaningful, we perform paired Wilcoxon signed-rank tests across the 47 datasets. For AUROC and AUPRC, paired differences are computed as RGLD minus ADERH, so positive values indicate that RGLD has higher detection accuracy. For runtime, paired differences are computed as ADERH runtime minus RGLD runtime, so positive values indicate that RGLD is faster. The results are shown in Table~\ref{tab:rgld_aderh_significance}.

\begin{table}[h]
\caption{Paired statistical comparison between RGLD and ADERH across 47 datasets. Positive differences favor RGLD.}
\centering
\begin{tabular}{lcccc}
\toprule
Metric & Median Diff. & Mean Diff. & $p$-value & Interpretation \\
\midrule
AUROC & -1.45 & -0.68 & 0.162 & Not significant \\
AUPRC & -1.42 & -1.13 & 0.228 & Not significant \\
Runtime (s) & 0.77 & 0.84 & $2.39 \times 10^{-9}$ & RGLD faster \\
\bottomrule
\end{tabular}
\label{tab:rgld_aderh_significance}
\end{table}

The paired tests show that ADERH has slightly higher mean AUROC and AUPRC, but the differences are not statistically significant at the conventional 0.05 level. By contrast, the runtime difference is highly significant, with RGLD being faster across the benchmark. These results support the main accuracy-efficiency claim of the paper: RGLD is statistically competitive with the strongest statistical baseline in detection accuracy, while offering substantially lower runtime.

The two methods also rely on different randomized inductive biases. ADERH builds an ensemble from randomized hypersphere-pair comparisons, whereas RGLD explicitly combines global random-feature density support with local projected-neighbor isolation across feature-bagged randomized views. Thus, RGLD provides a complementary randomized statistical approach: rather than relying on a single randomized scoring principle, it aggregates global and local density evidence across multiple views and scales.

\section{Full Pareto Frontier Graph}
\label{app:frontier}

\begin{figure}[H]
    \centering
    \includegraphics[width=0.8\linewidth]{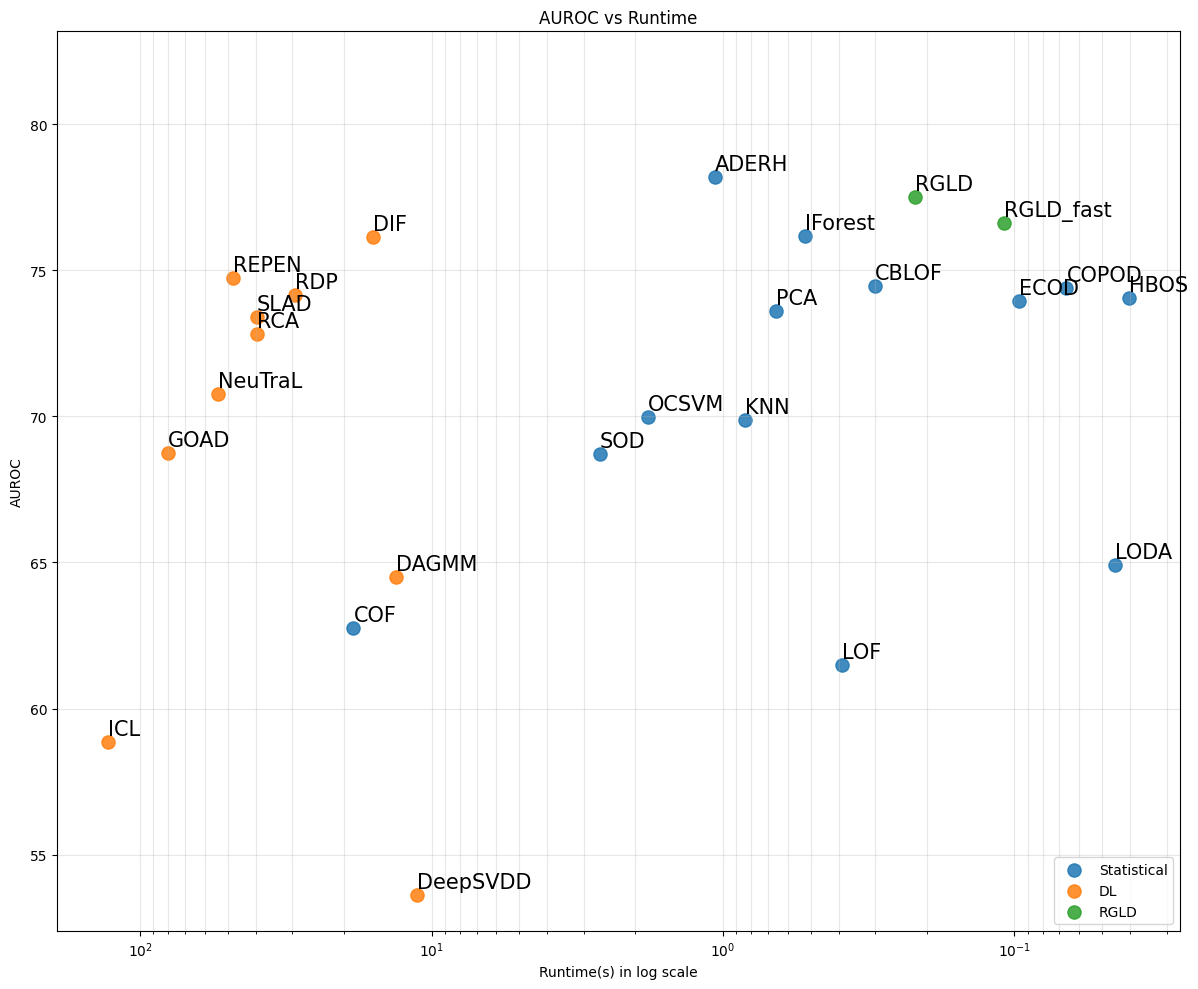}
    \caption{Full Pareto frontier graph of RGLD and 23 unsupervised anomaly detection methods using the average AUROC and time over 47 ADBench Datasets.}
\end{figure}

\begin{figure}[H]
    \centering
    \includegraphics[width=0.8\linewidth]{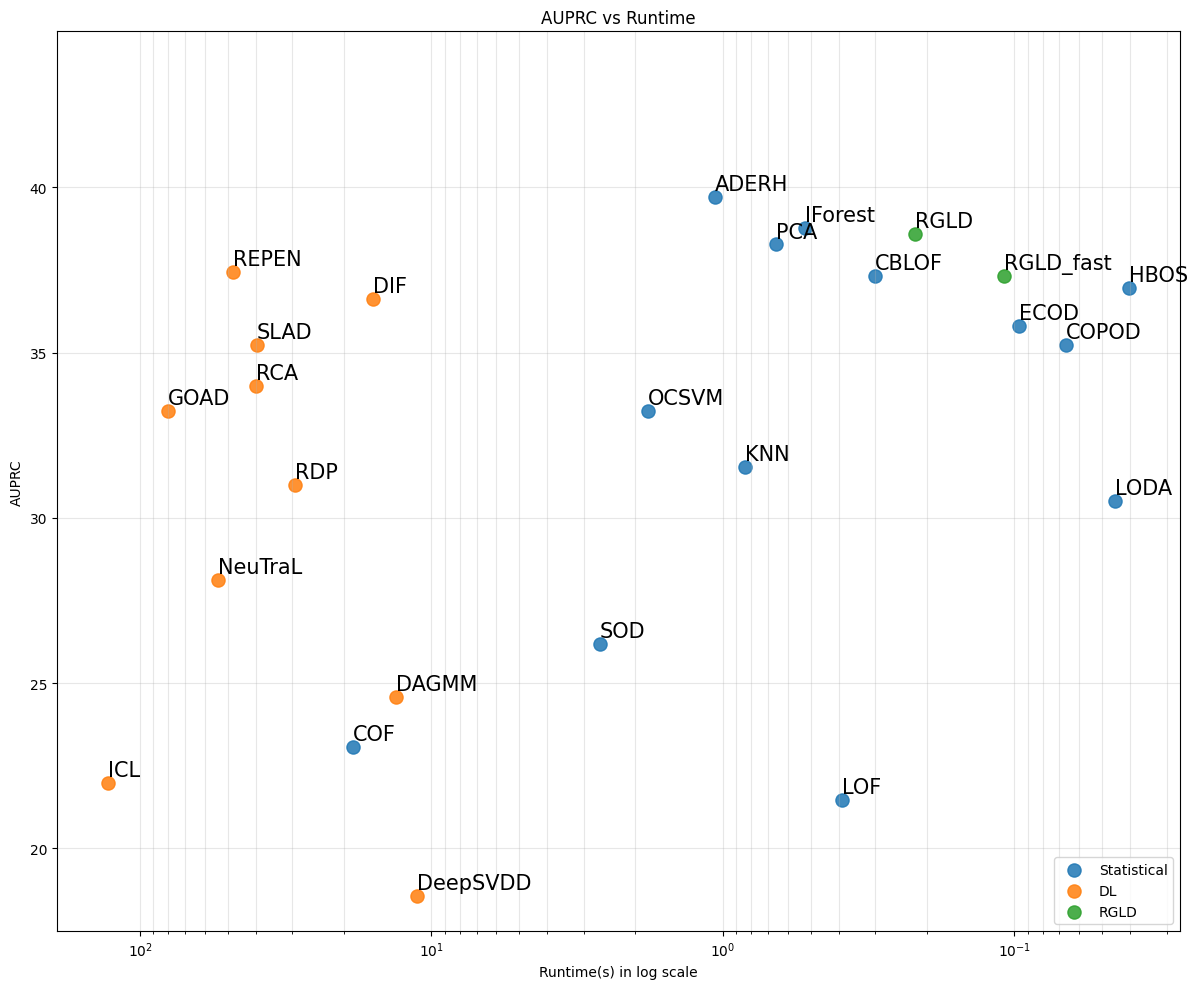}
    \caption{Full Pareto frontier graph of RGLD and 23 unsupervised anomaly detection methods using the average AUPRC and time over 47 ADBench Datasets.}
\end{figure}

\section{Dataset Detail}
\label{app:dataset}
\begin{center}
\setlength{\tabcolsep}{4pt}
\begin{tabular}{l l c c c c c c c}
\toprule
Dataset Name & Samples & \# of features & \% Anomaly  &  Category\\
\midrule
ALOI &49534& 27 & 3.04 & Image \\annthyroid &7200& 6 & 7.42& Healthcare \\backdoor & 95329 & 196 & 2.44 & Network \\breastw& 683 & 9 & 34.99 & Healthcare \\campaign &41188& 62& 11.27 & Finance
\\cardio & 1831 & 21 & 9.61 & Healthcare 
\\Cardiotocography & 2114 & 21 & 22.04 & Healthcare 
\\celeba & 202599 & 39 & 2.24 & Image 
\\census &299285 & 500 & 6.20 & Sociology 
\\cover &286048 & 10 & 0.96 & Botany 
\\donors &619326& 10& 5.93& Sociology
\\fault &1941& 27 & 34.67 & Physical 
\\fraud &284807& 29 & 0.17 & Finance
\\glass &214& 7 &4.21 & Forensic 
\\Hepatitis &80& 19& 16.25 & Healthcare
\\http &567498& 3 & 0.39 & Web 
\\InternetAds &1966& 1555 & 18.72 & Image 
\\Ionosphere &351& 32 & 35.90 & Oryctognosy 
\\landsat &6435& 36 &20.71 & Astronautics
\\letter &1600& 32& 6.25 & Image 
\\Lymphography &148& 18 & 4.05 & Healthcare 
\\magic.gamma &19020& 10 & 35.16 & Physical
\\mammography&11183& 6 &2.32 & Healthcare
\\mnist &7603& 100 & 9.21 & Image 
\\musk &3062& 166 & 3.17& Chemistry
\\optdigits &5216& 64 & 2.88 & Image 
\\PageBlocks &5393& 10 & 9.46 & Document
\\pendigits &6870& 16 & 2.27 & Image 
\\Pima &768& 8 & 34.90 & Healthcare 
\\satellite &6435& 36 & 31.64 & Astronautics
\\satimage-2 &5803& 36 & 1.22 & Astronautics
\\shuttle &49097& 9 &7.15 & Astronautics
\\skin &245057& 3 & 20.75 & Image 
\\smtp &95156&3 & 0.03 & Web 
\\SpamBase &4207& 57 & 39.91 & Document 
\\speech &	3686& 400 & 1.65 & Linguistics 
\\Stamps &340& 9 & 9.12 & 	Document 
\\thyroid &3772& 6 & 2.47 & Healthcare
\\vertebral &240& 6 & 12.50 & Biology 
\\vowels &1456& 12 & 3.43 & Linguistics 
\\Waveform &3443& 21 & 2.90 & Physics
\\WBC&223& 9 & 4.48 & Healthcare
\\WDBC &367& 30 & 2.72 & 	Healthcare
\\Wilt &4819& 5 & 5.33 & Botany
\\wine &129& 13 & 7.75 & Chemistry 
\\WPBC &198& 33& 23.74 & Healthcare
\\yeast &1484& 8 & 34.16 & Biology \\
\bottomrule
\end{tabular}
\end{center}

\section{Full Benchmark Results}
\label{app:benchmark}

This appendix reports the full per-dataset benchmark results across RGLD and 23 deep learning and statistical baselines. We report AUROC, AUPRC, and total runtime.

\begin{sidewaystable}[p]
\caption{Full benchmark AUROC results over 47 ADBench tabular datasets.}
\centering
\tiny
\setlength{\tabcolsep}{2.2pt}
\renewcommand{\arraystretch}{1.05}
\begin{adjustbox}{max width=\linewidth}
\begin{tabular}{lrrrrrrrrrrrrrrrrrrrrrrrr}
\toprule
Datasets & PCA & OCSVM & LOF & CBLOF & COF & HBOS & $k$NN & SOD & COPOD & ECOD & DAGMM & LODA & IForest & ADERH & DeepSVDD & DIF & RCA & GOAD & NeuTraL & ICL & SLAD & REPEN & RDP & RGLD \\
\midrule
ALOI & 56.65 & 55.85 & 66.63 & 55.22 & 64.68 & 52.63 & 61.47 & 61.09 & 53.75 & 56.6 & 51.96 & 51.33 & 56.66 & 54.39 & 50.29 & 56.48 & 54.38 & 57.08 & 56.71 & 52.07 & 53.67 & 57.39 & 56.99 & 51.79 \\
annthyroid & 66.25 & 57.23 & 70.2 & 62.26 & 65.92 & 60.15 & 71.69 & 77.38 & 76.8 & 78.03 & 56.53 & 41.02 & 82.01 & 63.96 & 76.62 & 66.34 & 60.55 & 57.91 & 62.98 & 37.7 & 75.88 & 65.23 & 62.2 & 93.53 \\
backdoor & 80.13 & 86.2 & 85.68 & 81.16 & 73.03 & 71.43 & 80.82 & 69.54 & 80.97 & 86.33 & 56.26 & 69.22 & 72.15 & 89.68 & 55.16 & 92.24 & 86.18 & 78.8 & 86.47 & 63.23 & 87.47 & 85.07 & 85.78 & 75.99 \\
breastw & 95.13 & 80.3 & 40.61 & 96.81 & 38.84 & 98.94 & 97.01 & 93.97 & 99.68 & 99.17 & N/A & 98.49 & 98.32 & 99.14 & 65.66 & 96.95 & 98.88 & 11.44 & 57.66 & 75.26 & 49.07 & 98.89 & 97.58 & 98.79 \\
campaign & 72.78 & 65.52 & 58.85 & 66.61 & 57.26 & 78.61 & 72.1 & 69.04 & 77.69 & 76.78 & 56.08 & 51.43 & 71.71 & 68.3 & 48.7 & 65.86 & 64.76 & 71.78 & 73.05 & 70.82 & 71.36 & 63.72 & 68.51 & 80.12 \\
cardio & 95.55 & 93.91 & 66.33 & 89.93 & 71.41 & 84.67 & 76.64 & 73.25 & 92.35 & 94.44 & 75.01 & 90.34 & 93.19 & 94.95 & 58.96 & 93.27 & 90.01 & 93.14 & 82.24 & 74.22 & 82.07 & 86.42 & 87.72 & 91.01 \\
Cardiotocography & 74.67 & 77.86 & 59.51 & 64.54 & 53.77 & 60.86 & 56.23 & 51.69 & 67.02 & 68.92 & 62.01 & 73.65 & 67.57 & 66.56 & 53.53 & 69 & 72.63 & 69.45 & 55.86 & 50.49 & 59.73 & 60.46 & 61.15 & 70.89 \\
celeba & 79.38 & 70.7 & 38.55 & 73.99 & 38.58 & 76.18 & 59.63 & 47.85 & 75.68 & 72.82 & 44.74 & 60.11 & 70.41 & 73.11 & 50.36 & 72.47 & 66.25 & 63.33 & 71.95 & 58.07 & 76.98 & 76.14 & 66.47 & 78.75 \\
census & 68.74 & 54.58 & 47.19 & 59.41 & 41.35 & 64.94 & 66.75 & 62.31 & 69.07 & 68.44 & 59.29 & 36.86 & 59.52 & 64.7 & 51.07 & 62.45 & 62.35 & 60.21 & 66.49 & 68.35 & 62.73 & 54.56 & 67.36 & 63.32 \\
cover & 93.73 & 92.62 & 84.58 & 89.3 & 76.91 & 80.24 & 85.97 & 74.46 & 88.64 & 93.42 & 89.89 & 92.34 & 86.74 & 87.23 & 46.2 & 92.63 & 81.47 & 95.17 & 92.89 & 65.98 & 91.19 & 90.32 & 94.98 & 83.22 \\
donors & 83.15 & 71.93 & 55.49 & 60.44 & 70.54 & 78.23 & 81.09 & 55.21 & 81.76 & 74.45 & 70.57 & 24.86 & 77.68 & 76.1 & 50.27 & 72.11 & 82.19 & 46.47 & 54.66 & 60.02 & 50.48 & 74.23 & 76.28 & 85.95 \\
fault & 46.02 & 47.69 & 58.93 & 64.06 & 62.1 & 51.28 & 72.98 & 68.11 & 43.88 & 43.41 & 45.86 & 41.71 & 57.02 & 70.25 & 51.67 & 67.11 & 61.6 & 52.94 & 69.69 & 68.62 & 66.63 & 65.69 & 70.53 & 49.93 \\
fraud & 90.35 & 90.62 & 94.92 & 91.7 & 93.05 & 90.29 & 93.56 & 94.97 & 88.32 & 89.85 & 89.53 & 88.99 & 90.38 & 92.08 & 64.98 & 91.3 & 86.54 & 99.69 & 98.48 & 99.71 & 99.71 & 98.98 & 99.31 & 90.06 \\
glass & 66.29 & 35.36 & 69.2 & 82.94 & 72.24 & 77.23 & 82.29 & 73.36 & 72.43 & 75.7 & 76.09 & 73.13 & 77.13 & 84.72 & 47.49 & 83.95 & 62.33 & 79.84 & 82.41 & 61.25 & 87.67 & 78.06 & 87.35 & 80.85 \\
Hepatitis & 75.95 & 67.75 & 38.02 & 66.4 & 41.45 & 79.85 & 52.76 & 68.17 & 82.05 & 79.67 & 54.8 & 64.87 & 69.75 & 74.31 & 50.96 & 67.13 & 59.48 & 72.31 & 59.14 & 26.74 & 67.38 & 45.91 & 60.66 & 77.82 \\
http & 99.72 & 99.59 & 27.46 & 99.6 & 88.78 & 99.53 & 3.37 & 78.04 & 99.29 & 98.1 & N/A & 12.48 & 99.96 & 99.52 & 69.05 & 99.59 & 99.49 & 98.99 & 98.7 & 49.98 & 99.73 & 99.06 & 99.1 & 99.99 \\
InternetAds & 61.67 & 68.28 & 65.83 & 70.58 & 63.79 & 68.03 & 69.99 & 61.85 & 67.05 & 67.1 & N/A & 55.38 & 69.01 & 66.85 & 60.2 & 68.67 & 70.45 & 46.46 & 56.04 & 53.91 & 59.09 & 67.75 & 63.84 & 70.76 \\
Ionosphere & 79.19 & 75.92 & 90.59 & 90.72 & 86.76 & 62.49 & 88.26 & 86.41 & 79.34 & 75.59 & 73.41 & 78.42 & 84.5 & 89.94 & 50.89 & 90.58 & 77.8 & 85.27 & 93.73 & 33.78 & 95.2 & 80.4 & 92.6 & 88.35 \\
landsat & 35.76 & 36.15 & 53.9 & 63.55 & 53.5 & 55.14 & 57.95 & 59.54 & 41.55 & 56.61 & 43.92 & 38.17 & 47.64 & 59.56 & 63.61 & 50.5 & 29.63 & 45.84 & 57.44 & 65.76 & 64.25 & 55.09 & 47.63 & 50.35 \\
letter & 50.29 & 46.18 & 84.49 & 75.62 & 80.03 & 59.74 & 86.19 & 84.09 & 54.32 & 50.76 & 50.42 & 50.24 & 61.07 & 72.72 & 56.64 & 68.38 & 58.74 & 52.39 & 83.86 & 72.01 & 73.39 & 67.87 & 76.88 & 65.65 \\
Lymphography & 99.82 & 99.54 & 89.86 & 99.83 & 90.85 & 99.49 & 55.91 & 72.49 & 99.48 & 99.52 & 72.11 & 85.55 & 99.81 & 99.89 & 32.29 & 99.78 & 99.85 & 92.86 & 75.94 & 64.74 & 73.95 & 100 & 81.88 & 99 \\
magic.gamma & 67.22 & 60.65 & 68.51 & 75.13 & 66.64 & 70.86 & 82.38 & 75.4 & 68.33 & 64.36 & 58.58 & 68.02 & 73.25 & 80.94 & 60.26 & 76.82 & 74.51 & 65.63 & 65.8 & 64.26 & 68.5 & 67.82 & 75.23 & 69.04 \\
mammography & 88.72 & 84.95 & 74.39 & 83.74 & 77.53 & 86.27 & 84.53 & 81.51 & 90.69 & 90.75 & N/A & 83.91 & 86.39 & 82.82 & 56.98 & 82.53 & 85.54 & 29.24 & 36.63 & 40.31 & 26.41 & 63.93 & 79.65 & 71.45 \\
mnist & 85.29 & 82.95 & 67.13 & 79.45 & 70.78 & 60.42 & 80.58 & 60.1 & 77.74 & 84.6 & 67.23 & 72.27 & 80.98 & 86.99 & 53.4 & 79.24 & 72.77 & 88.44 & 74.61 & 83.85 & 93.13 & 79.6 & 79.23 & 65.8 \\
musk & 100 & 80.58 & 41.18 & 100 & 38.69 & 100 & 69.89 & 74.09 & 94.2 & 95.11 & 76.85 & 95.11 & 99.99 & 100 & 43.52 & 99.96 & 99.15 & 99.36 & 97.74 & 100 & 99.99 & 100 & 91.53 & 97.82 \\
optdigits & 51.65 & 54 & 56.1 & 87.51 & 49.15 & 81.63 & 41.73 & 58.92 & 68.71 & 61.04 & 62.57 & 61.74 & 70.92 & 76.74 & 38.89 & 60.15 & 78.06 & 48.58 & 89.38 & 62.3 & 61.12 & 43.88 & 63.3 & 66.31 \\
PageBlocks & 90.64 & 88.76 & 75.9 & 85.04 & 72.65 & 80.58 & 81.94 & 77.75 & 88.05 & 90.92 & 89.61 & 83.34 & 89.57 & 89.37 & 57.77 & 90.34 & 71.31 & 88.46 & 89.41 & 73.24 & 93.4 & 88.23 & 84.02 & 91.23 \\
pendigits & 93.73 & 93.75 & 47.99 & 90.4 & 45.07 & 93.04 & 72.95 & 66.29 & 90.68 & 91.22 & 64.22 & 89.1 & 94.76 & 96.5 & 39.92 & 95.32 & 88.53 & 93.77 & 85.99 & 50.6 & 91.55 & 93.33 & 82.07 & 97.61 \\
Pima & 70.77 & 66.92 & 65.71 & 71.42 & 61.05 & 71.07 & 73.43 & 61.25 & 69.1 & 51.54 & 55.92 & 65.93 & 72.87 & 78.36 & 51.03 & 75.2 & 77.63 & 59.81 & 57.49 & 47.19 & 56.98 & 60.6 & 65.31 & 73.58 \\
satellite & 59.62 & 59.02 & 55.88 & 71.32 & 54.74 & 74.8 & 65.18 & 63.96 & 63.2 & 75.06 & 62.33 & 61.98 & 70.43 & 78.85 & 55.3 & 70.31 & 53.36 & 68.74 & 72.45 & 53.12 & 70.1 & 71.97 & 63.17 & 71.9 \\
satimage-2 & 97.62 & 97.35 & 47.36 & 99.84 & 56.7 & 97.65 & 92.6 & 83.08 & 97.21 & 97.11 & 96.29 & 97.56 & 99.16 & 99.87 & 53.14 & 99.84 & 98.36 & 98.23 & 95.92 & 66.8 & 94.27 & 99.26 & 98.19 & 99.37 \\
shuttle & 98.62 & 97.4 & 57.11 & 83.48 & 51.72 & 98.63 & 69.64 & 69.51 & 99.35 & 99.4 & 97.92 & 60.95 & 99.56 & 98.27 & 52.05 & 96.91 & 97.95 & 99.17 & 85.64 & 36.07 & 99.37 & 99.46 & 97.49 & 98.22 \\
skin & 45.26 & 49.45 & 46.47 & 69.49 & 41.66 & 60.15 & 71.46 & 60.35 & 47.55 & 39.09 & N/A & 45.75 & 68.21 & 78.89 & 44.05 & 77.98 & 69.83 & 56.4 & 43.8 & 42.75 & 81.03 & 63.85 & 70.21 & 75.55 \\
smtp & 88.41 & 80.7 & 71.84 & 79.68 & 79.6 & 70.52 & 89.62 & 59.85 & 79.09 & 71.86 & 71.32 & 67.43 & 89.73 & 78.89 & 78.24 & 80.74 & 63.14 & 100 & 100 & 50 & 100 & 99.93 & 100 & 94.31 \\
SpamBase & 54.66 & 52.47 & 43.33 & 54.97 & 40.96 & 64.74 & 53.35 & 52.35 & 70.09 & 66.89 & N/A & 41.99 & 64.76 & 58.7 & 53.55 & 56.04 & 55.83 & 19.86 & 35.64 & 42.83 & 29.76 & 45.24 & 58.39 & 73.52 \\
speech & 50.79 & 50.19 & 52.48 & 50.58 & 55.97 & 50.59 & 51.03 & 55.86 & 52.89 & 51.58 & 52.75 & 49.84 & 50.74 & 53.03 & 53.43 & 51.35 & 49.96 & 53.43 & 55.34 & 41.14 & 55.6 & 55.15 & 53.45 & 51.26 \\
Stamps & 91.47 & 83.86 & 51.26 & 68.18 & 53.81 & 90.73 & 68.61 & 73.26 & 93.4 & 91.41 & 88.88 & 87.18 & 91.21 & 91.18 & 55.84 & 86.7 & 75.13 & 87.52 & 75.23 & 46.97 & 91.5 & 86.84 & 72.49 & 89.73 \\
thyroid & 96.34 & 87.92 & 86.86 & 94.73 & 90.87 & 95.62 & 95.93 & 92.81 & 94.3 & 97.78 & 79.75 & 74.3 & 98.3 & 94.9 & 49.64 & 96.01 & 86.4 & 94.56 & 98.34 & 73.63 & 98.89 & 95.2 & 94.26 & 98.99 \\
vertebral & 37.06 & 37.99 & 49.29 & 41.41 & 48.71 & 28.56 & 33.79 & 40.32 & 25.64 & 37.51 & 53.2 & 30.57 & 36.66 & 35.48 & 36.67 & 35.34 & 41.39 & 50.57 & 56.62 & 45.19 & 52.37 & 45.38 & 52.51 & 32.22 \\
vowels & 65.29 & 61.59 & 93.12 & 89.92 & 94.04 & 72.21 & 97.26 & 92.65 & 53.15 & 45.81 & 60.58 & 70.36 & 73.94 & 93.37 & 52.49 & 88.53 & 80.45 & 53.78 & 85.5 & 75.89 & 90.16 & 91.41 & 89.51 & 80.5 \\
Waveform & 65.48 & 56.29 & 73.32 & 72.42 & 72.56 & 68.77 & 73.78 & 68.57 & 75.03 & 73.25 & 49.35 & 60.13 & 71.47 & 80.93 & 54.47 & 73.59 & 73.43 & 63.84 & 58.06 & 44.72 & 70.84 & 80.99 & 52.51 & 76.66 \\
WBC & 98.2 & 99.03 & 54.17 & 99.46 & 60.9 & 98.72 & 90.56 & 94.6 & 99.11 & 99.11 & N/A & 96.91 & 99.01 & 99.24 & 55.5 & 95.33 & 97.97 & 93.16 & 37.09 & 55.32 & 63.54 & 99.6 & 95.13 & 98.49 \\
WDBC & 99.05 & 98.86 & 89 & 99.32 & 96.26 & 99.5 & 91.72 & 91.9 & 99.42 & 97.2 & 76.67 & 98.26 & 98.95 & 98.34 & 65.69 & 96.57 & 96.72 & 97.71 & 78.16 & 85.32 & 87.27 & 99.37 & 93.32 & 99.06 \\
Wilt & 20.39 & 31.28 & 50.65 & 32.54 & 49.66 & 32.49 & 48.42 & 53.25 & 33.4 & 39.43 & 37.29 & 26.42 & 41.94 & 30.44 & 46.08 & 30.03 & 38.27 & 35.9 & 81.08 & 49.74 & 56.07 & 38.22 & 30.47 & 59.94 \\
wine & 84.37 & 73.07 & 37.74 & 25.86 & 44.44 & 91.36 & 44.98 & 46.11 & 88.65 & 71.34 & 61.7 & 90.12 & 80.37 & 78.55 & 59.52 & 70.11 & 82.72 & 68 & 42.07 & 58.06 & 75.44 & 84.6 & 48.54 & 79.01 \\
WPBC & 46.01 & 45.35 & 41.41 & 44.77 & 45.88 & 51.24 & 46.59 & 51.14 & 49.34 & 46.83 & 47.8 & 49.31 & 46.63 & 49.46 & 49.79 & 47.3 & 47.38 & 44.46 & 48.77 & 57.41 & 55.22 & 44.43 & 48.38 & 47.47 \\
yeast & 41.15 & 41 & 45.31 & 44.85 & 44.48 & 39.64 & 39.06 & 42.46 & 36.99 & 39.61 & 41.11 & 44.58 & 37.76 & 36.17 & 47.92 & 38.85 & 41.17 & 40.5 & 42.4 & 46.33 & 39.16 & 42.33 & 42.1 & 37.13 \\
\bottomrule
\end{tabular}
\end{adjustbox}
\label{tab:appendix_auroc}
\end{sidewaystable}

\begin{sidewaystable}[p]
\caption{Full benchmark AUPRC results over 47 ADBench tabular datasets.}
\centering
\tiny
\setlength{\tabcolsep}{2.2pt}
\renewcommand{\arraystretch}{1.05}
\begin{adjustbox}{max width=\linewidth}
\begin{tabular}{lrrrrrrrrrrrrrrrrrrrrrrrr}
\toprule
Datasets & PCA & OCSVM & LOF & CBLOF & COF & HBOS & $k$NN & SOD & COPOD & ECOD & DAGMM & LODA & IForest & ADERH & DeepSVDD & DIF & RCA & GOAD & NeuTraL & ICL & SLAD & REPEN & RDP & RGLD \\
\midrule
ALOI & 4.17 & 5.02 & 8.08 & 4.46 & 6.85 & 3.69 & 6.02 & 5.97 & 3.62 & 3.9 & 4.33 & 4.53 & 3.9 & 4.75 & 4.01 & 4.47 & 4.67 & 6.44 & 7.66 & 3.54 & 5.55 & 5.88 & 6.7 & 3.33 \\
annthyroid & 16.12 & 10.37 & 15.71 & 13.69 & 14.39 & 16.99 & 16.74 & 18.84 & 16.58 & 24.65 & 9.64 & 7.06 & 30.47 & 13.59 & 21.95 & 16.44 & 13.21 & 11.46 & 16.34 & 10.79 & 29.53 & 13.97 & 14.62 & 48.96 \\
backdoor & 31.29 & 9.69 & 26.14 & 6.96 & 24.68 & 4.91 & 45.22 & 39.41 & 7.69 & 11.25 & 6.5 & 14.51 & 4.75 & 20.78 & 12.85 & 25.5 & 9.49 & 10.52 & 48.79 & 32.69 & 44.28 & 10.57 & 19.88 & 5.17 \\
breastw & 95.11 & 82.7 & 28.55 & 91.54 & 27.6 & 97.71 & 92.19 & 84.88 & 99.4 & 98.54 & N/A & 97.04 & 96.04 & 98.07 & 50.92 & 88.97 & 97.5 & 20.21 & 34.8 & 59.54 & 30.65 & 97.14 & 92.24 & 97.25 \\
campaign & 27.9 & 29.22 & 14.51 & 23.99 & 13.01 & 37.99 & 27.18 & 18.88 & 38.58 & 37.4 & 14.62 & 13.47 & 32.26 & 27.76 & 11.6 & 22.85 & 22.35 & 32.77 & 30.35 & 26.86 & 30.82 & 26.71 & 26.52 & 35.45 \\
cardio & 66.06 & 62.89 & 23.79 & 61.95 & 28.67 & 52.1 & 40.72 & 28.54 & 60.42 & 68.59 & 28.92 & 53.41 & 59.95 & 64.25 & 22.5 & 55.18 & 53.4 & 55.86 & 39.41 & 22.21 & 39.89 & 47.28 & 41.17 & 45.45 \\
Cardiotocography & 47.95 & 52.61 & 30.66 & 45.44 & 28.21 & 38.28 & 34.79 & 27.99 & 40.46 & 43.57 & 30.61 & 48 & 41.47 & 44.56 & 34.03 & 42.79 & 44.52 & 46.61 & 37.63 & 31.8 & 43.25 & 40.28 & 38.7 & 40.02 \\
celeba & 15.89 & 10.73 & 1.71 & 11.33 & 1.77 & 13.82 & 3.14 & 2.66 & 13.69 & 12.37 & 1.95 & 4.04 & 8.96 & 7.14 & 2.34 & 9.08 & 5.05 & 3.67 & 4.37 & 2.61 & 5.07 & 5.71 & 3.88 & 13.07 \\
census & 10.02 & 6.76 & 5.45 & 7.44 & 4.88 & 8.69 & 9 & 8.52 & 9.92 & 9.72 & 8.71 & 5.01 & 7.78 & 8.32 & 6.87 & 8.04 & 8.17 & 7.46 & 10.09 & 10.8 & 8.23 & 6.93 & 10.33 & 8.31 \\
cover & 9.8 & 11.41 & 8.12 & 5.83 & 4 & 6.83 & 6.16 & 3.88 & 11.37 & 15.63 & 27.59 & 13.06 & 8.85 & 6.15 & 8.12 & 10 & 3.68 & 12.62 & 12.49 & 1.88 & 12.46 & 5.89 & 20.73 & 7.34 \\
donors & 17.9 & 9.86 & 7.88 & 6.89 & 8.8 & 23.36 & 14.75 & 9.69 & 21.58 & 14.17 & 10.53 & 3.78 & 12.74 & 11.48 & 6.38 & 10.09 & 17.06 & 5.2 & 6.07 & 6.77 & 5.63 & 10.14 & 11.43 & 25.2 \\
fault & 32.76 & 38.44 & 38.38 & 43.98 & 41.56 & 36.47 & 54.45 & 48.01 & 30.54 & 30.82 & 33.48 & 31.03 & 41.09 & 54.21 & 39.15 & 51.61 & 45.49 & 46.74 & 54.42 & 50.77 & 54.91 & 52.22 & 54.07 & 33.09 \\
fraud & 22.91 & 47.58 & 47.4 & 47.52 & 22.86 & 25.89 & 47.3 & 31.37 & 42.82 & 42.99 & 21.32 & 46.37 & 21.67 & 26.29 & 8.97 & 38.55 & 32.38 & 60.49 & 23.17 & 41.52 & 57.74 & 54.75 & 42.3 & 26.53 \\
glass & 10.05 & 8.02 & 20.11 & 13.84 & 11.81 & 11.82 & 20.26 & 18.73 & 9.78 & 18.43 & 24.58 & 13.37 & 10.99 & 15.05 & 8.72 & 15.1 & 6.86 & 26.55 & 14.23 & 10.6 & 20.65 & 15.57 & 22.29 & 11.48 \\
Hepatitis & 36.65 & 29.44 & 13.69 & 31.54 & 14.39 & 37.73 & 21.95 & 24.89 & 41.5 & 37.82 & 22.96 & 30.9 & 26.25 & 28.45 & 22.17 & 24.87 & 18.02 & 28.72 & 26.01 & 12.16 & 24.03 & 15.98 & 31.75 & 31.68 \\
http & 56.43 & 46.86 & 3.82 & 47.53 & 9.57 & 44.79 & 0.7 & 8.32 & 35.19 & 16.61 & N/A & 0.67 & 90.83 & 43.52 & 29.3 & 48.87 & 26.83 & 31.33 & 26.12 & 0.5 & 61.61 & 33.44 & 21.65 & 95.69 \\
InternetAds & 32.55 & 54.68 & 40.49 & 58.13 & 38.67 & 53.97 & 43.23 & 27.69 & 50.97 & 51.07 & N/A & 23.89 & 48.6 & 49.94 & 27.91 & 52.12 & 57.87 & 16.13 & 20.55 & 20.77 & 22.93 & 47.15 & 34.97 & 55.38 \\
Ionosphere & 73.92 & 74.54 & 88.06 & 90.27 & 82.91 & 41.78 & 90.41 & 85.87 & 69.89 & 65.99 & 64.98 & 73.04 & 80.41 & 88.2 & 41.79 & 88.76 & 75.01 & 76.13 & 92.31 & 35.6 & 92.42 & 69.43 & 90.26 & 85.49 \\
landsat & 16.18 & 16.21 & 24.69 & 30.97 & 24.95 & 22.03 & 24.65 & 26.38 & 17.48 & 25.17 & 24.48 & 18.86 & 19.81 & 22.99 & 38.83 & 20.58 & 13.98 & 27.02 & 28.09 & 38.58 & 52.35 & 22.94 & 20.03 & 21 \\
letter & 6.86 & 6.1 & 34.02 & 14.8 & 21.43 & 8.38 & 30 & 28.63 & 6.77 & 6.94 & 11.68 & 6.87 & 8.49 & 12.76 & 9.29 & 12.03 & 8.27 & 12.87 & 32.23 & 19.36 & 20.97 & 12 & 15.87 & 10.72 \\
Lymphography & 97.02 & 93.59 & 23.08 & 97.62 & 36.68 & 91.83 & 38.69 & 22.65 & 88.68 & 90.87 & 19.52 & 44.54 & 97.31 & 98.08 & 4.58 & 95.24 & 97.19 & 54.72 & 11.02 & 11.9 & 10.14 & 100 & 19.25 & 83.24 \\
magic.gamma & 59.27 & 51.43 & 54.76 & 68.85 & 54.12 & 62.41 & 75.63 & 67.89 & 59.18 & 54.38 & 46.92 & 58.49 & 64.72 & 73.51 & 49.17 & 69.5 & 64.8 & 59.17 & 56.15 & 52.79 & 65.01 & 63.78 & 68.15 & 61.97 \\
mammography & 19.25 & 12.94 & 9.8 & 11.14 & 11.14 & 21.31 & 15.91 & 13.41 & 40.67 & 41.28 & N/A & 14.75 & 20.67 & 10.38 & 6.26 & 11.76 & 18.8 & 1.52 & 1.67 & 3.46 & 1.45 & 3.41 & 10.99 & 5.75 \\
mnist & 39.93 & 33.2 & 20.9 & 28.82 & 25.51 & 12.51 & 35.53 & 19.15 & 21.35 & 31.93 & 23.75 & 25.86 & 27.71 & 37.91 & 19.72 & 24.59 & 27.13 & 46.2 & 33.04 & 34.42 & 55.84 & 37.68 & 31.54 & 14.37 \\
musk & 99.89 & 10.61 & 2.82 & 100 & 2.61 & 100 & 9.65 & 7.59 & 34.79 & 34.95 & 32.75 & 47.6 & 99.61 & 100 & 5.39 & 98.21 & 85.36 & 88.36 & 61.07 & 100 & 99.64 & 100 & 37.68 & 65.16 \\
optdigits & 2.76 & 2.92 & 6.06 & 10.08 & 4.42 & 10.03 & 3.06 & 4.39 & 4.36 & 3.43 & 5.59 & 3.95 & 5.09 & 5.88 & 2.5 & 3.65 & 7.33 & 2.61 & 12.05 & 4.44 & 3.67 & 2.42 & 4.08 & 4.46 \\
PageBlocks & 51.71 & 49.14 & 39.64 & 49.65 & 41.02 & 33.32 & 45.39 & 37.83 & 37.65 & 49.3 & 53.25 & 51.29 & 46.04 & 53.37 & 31.45 & 52.47 & 47.12 & 49.78 & 52.38 & 29.02 & 58.44 & 53.89 & 47.98 & 51.29 \\
pendigits & 23.65 & 23.52 & 3.78 & 17.27 & 2.89 & 29.27 & 6.5 & 4.46 & 21.22 & 23.07 & 4.67 & 18.71 & 26.05 & 30.59 & 2.45 & 23.91 & 10.4 & 33.76 & 8.69 & 2.87 & 12.81 & 36.69 & 7.69 & 49.47 \\
Pima & 54.03 & 50 & 47.18 & 53.19 & 44.7 & 56.61 & 55.14 & 48.24 & 55.19 & 37.3 & 41.55 & 44.09 & 55.82 & 59.35 & 35.87 & 56.77 & 59.54 & 47.9 & 45.89 & 37.57 & 46.28 & 52.38 & 50.86 & 57.54 \\
satellite & 59.64 & 57.61 & 37.68 & 61.48 & 39.7 & 67.25 & 50.01 & 47.23 & 56.58 & 65.94 & 58.33 & 61.94 & 65.92 & 72.26 & 40.11 & 59.94 & 43.59 & 69.61 & 60.96 & 45.59 & 63.46 & 68.37 & 50.76 & 68.61 \\
satimage-2 & 85.69 & 82.71 & 4.29 & 97.09 & 8.8 & 78.04 & 39.14 & 26.11 & 76.55 & 63.25 & 22.07 & 80.52 & 93.45 & 97.47 & 3.08 & 94.61 & 95.25 & 74.65 & 19.49 & 4.36 & 10.09 & 91.28 & 41.82 & 92.93 \\
shuttle & 92.35 & 85.29 & 13.76 & 60.98 & 12.17 & 96.4 & 20.38 & 20.27 & 96.56 & 95.76 & 93.2 & 48.75 & 97.62 & 89.31 & 15.86 & 65.68 & 95.11 & 95.93 & 43.82 & 8.41 & 97.95 & 97.95 & 84.38 & 89.06 \\
skin & 17.4 & 19.03 & 18.25 & 29.82 & 16.38 & 23.7 & 28.72 & 24.61 & 17.99 & 15.96 & N/A & 18.44 & 26.08 & 34.65 & 18.48 & 33.95 & 30.02 & 21.31 & 20.28 & 20.71 & 36.84 & 23.83 & 27.55 & 31.55 \\
smtp & 66.7 & 18.9 & 20.69 & 26.68 & 61.13 & 35.2 & 66.7 & 33.36 & 1.08 & 50.01 & 50.03 & 35.77 & 1.24 & 50.02 & 50.02 & 50.02 & 18.52 & 100 & 100 & 0.03 & 100 & 33.33 & 100 & 4.16 \\
SpamBase & 41.57 & 40.12 & 35.16 & 41.18 & 34.73 & 50.03 & 41.42 & 40.03 & 56.68 & 53.95 & N/A & 35.88 & 51.75 & 43.15 & 42.23 & 42.07 & 41.97 & 26.8 & 32.26 & 34.76 & 30.05 & 37.03 & 43.39 & 65.1 \\
speech & 1.97 & 1.96 & 2.52 & 1.99 & 2.25 & 2.09 & 2.02 & 2.13 & 1.94 & 1.77 & 2.03 & 1.79 & 2.31 & 2.02 & 5.12 & 2.33 & 3.05 & 2.63 & 2.13 & 1.39 & 1.9 & 2.2 & 3.33 & 1.99 \\
Stamps & 41.09 & 31.39 & 21.29 & 23.66 & 16.5 & 35.24 & 23.53 & 20.28 & 43.1 & 38.17 & 43.72 & 34.6 & 39.49 & 39.56 & 11.4 & 33.23 & 20.44 & 26.24 & 16.61 & 12.04 & 33.57 & 24.48 & 15.34 & 32.87 \\
thyroid & 44.34 & 21.23 & 20.81 & 29.95 & 28.5 & 50.98 & 34.98 & 23.56 & 19.64 & 54.05 & 16.06 & 14.68 & 63.11 & 31.12 & 2.5 & 41.26 & 26.72 & 36.09 & 57.2 & 46.1 & 82.81 & 31.63 & 38.54 & 67.46 \\
vertebral & 10.49 & 10.94 & 14.24 & 11.58 & 13.85 & 9.23 & 10.57 & 11.79 & 8.89 & 11.24 & 15.24 & 9.68 & 10.46 & 10.28 & 10.49 & 10.29 & 10.54 & 11.44 & 14.62 & 11.99 & 13.7 & 9.6 & 12.72 & 9.34 \\
vowels & 8.92 & 8.24 & 34.42 & 22.12 & 55.96 & 13.41 & 63.41 & 38.88 & 4.14 & 3.92 & 12.22 & 13.82 & 15.12 & 43.18 & 4.99 & 33.72 & 20.37 & 11.09 & 14.72 & 16.98 & 28.95 & 22.6 & 20.61 & 13.41 \\
Waveform & 5.79 & 4.37 & 11.33 & 18.98 & 14.11 & 5.86 & 13.04 & 9.66 & 6.9 & 6.86 & 3.11 & 4.71 & 6.24 & 14.35 & 4.83 & 7.83 & 8.44 & 4.03 & 4.22 & 2.9 & 38.45 & 22.03 & 3.27 & 10.29 \\
WBC & 82.29 & 89.87 & 5.57 & 92.27 & 9.73 & 73.56 & 66.55 & 54 & 86.19 & 86.19 & N/A & 78.67 & 90.49 & 90.54 & 6.38 & 47.61 & 77.4 & 40.24 & 4.09 & 7.74 & 6.67 & 94.04 & 32.6 & 80.44 \\
WDBC & 75.46 & 71.88 & 14.93 & 79.62 & 50.52 & 88.98 & 43.72 & 35.6 & 84.78 & 57.91 & 18.48 & 66.11 & 78.53 & 59.26 & 6.57 & 47.97 & 44.04 & 48.58 & 5.96 & 21.61 & 12.64 & 76.26 & 18.51 & 71.47 \\
Wilt & 3.13 & 3.62 & 5.05 & 3.64 & 4.98 & 3.84 & 4.73 & 5.53 & 3.69 & 4.14 & 4 & 3.36 & 4.23 & 3.53 & 4.67 & 3.51 & 3.99 & 3.95 & 12.84 & 5.33 & 5.66 & 3.99 & 3.52 & 6.04 \\
wine & 30.87 & 21.56 & 7.77 & 5.83 & 8.45 & 43.08 & 8.43 & 7.95 & 45.71 & 18.37 & 17.51 & 48.82 & 25.96 & 17.94 & 21.14 & 14.17 & 23.73 & 13.7 & 6.78 & 10.84 & 14.27 & 24.67 & 7.23 & 22.41 \\
WPBC & 23.01 & 22.93 & 20.29 & 21.32 & 21.3 & 23.04 & 21.49 & 25.37 & 22.81 & 21.38 & 22.49 & 25.58 & 22.42 & 22.88 & 26.24 & 22.4 & 23.77 & 20.79 & 25.64 & 34.09 & 32.75 & 20.83 & 22.2 & 21.64 \\
yeast & 29.9 & 29.84 & 31.64 & 30.93 & 31.27 & 32.75 & 29.33 & 29.96 & 30.71 & 31.36 & 29.92 & 33.29 & 29.8 & 27.81 & 33.03 & 28.5 & 29.63 & 32.73 & 32.97 & 32.02 & 29.94 & 31.32 & 33.34 & 30.88 \\
\bottomrule
\end{tabular}
\end{adjustbox}
\label{tab:appendix_auprc}
\end{sidewaystable}

\begin{sidewaystable}[p]
\caption{Full benchmark Runtime (s) results over 47 ADBench tabular datasets.}
\centering
\tiny
\setlength{\tabcolsep}{2.2pt}
\renewcommand{\arraystretch}{1.05}
\begin{adjustbox}{max width=\linewidth}
\begin{tabular}{lrrrrrrrrrrrrrrrrrrrrrrrr}
\toprule
Datasets & PCA & OCSVM & LOF & CBLOF & COF & HBOS & $k$NN & SOD & COPOD & ECOD & DAGMM & LODA & IForest & ADERH & DeepSVDD & DIF & RCA & GOAD & NeuTraL & ICL & SLAD & REPEN & RDP & RGLD \\
\midrule
ALOI & 0.23 & 3.86 & 1.65 & 0.83 & 27.69 & 0.94 & 1.62 & 6.05 & 0.1 & 0.13 & 31.17 & 0.06 & 0.39 & 1.56 & 20.55 & 34.42 & 95.16 & 159.27 & 137.83 & 158.46 & 37.88 & 100.33 & 43.15 & 0.29 \\
annthyroid & 0.01 & 1.18 & 0.17 & 0.39 & 5.13 & 0 & 0.19 & 2.25 & 0.02 & 0.03 & 17.89 & 0.05 & 0.27 & 1.23 & 12.82 & 20.38 & 63.47 & 110.87 & 84.46 & 268.3 & 25.93 & 67.19 & 36.41 & 0.21 \\
backdoor & 0.73 & 6.43 & 1.85 & 0.14 & 121.26 & 0.06 & 3.51 & 5.15 & N/A & N/A & 24.8 & 0.08 & 2.02 & 2.15 & 19.25 & 33.8 & 63.06 & 162.97 & 118.98 & 60.04 & 40.91 & 100.84 & 51.66 & 0.5 \\
breastw & 0 & 0.03 & 0.01 & 0.12 & 0.15 & 0.01 & 0.02 & 0.13 & 0.01 & 0 & N/A & 0.02 & 0.12 & 0.2 & 1.79 & 2.42 & 8.93 & 15.78 & 11.05 & 38.06 & 48.75 & 7.63 & 7.08 & 0.06 \\
campaign & 0.15 & 4.38 & 1.71 & 0.29 & 45.75 & 0.02 & 2.3 & 4.85 & 0.13 & 0.19 & 18.41 & 0.06 & 0.79 & 1.38 & 24.39 & 33.6 & 82.43 & 155.97 & 118.96 & 118.81 & 52.13 & 101.46 & 78.8 & 0.37 \\
cardio & 0.02 & 0.1 & 0.07 & 0.23 & 0.38 & 0 & 0.07 & 0.28 & 0.02 & 0.02 & 3.4 & 0.04 & 0.16 & 0.4 & 4.98 & 6.46 & 17.53 & 28.82 & 25.31 & 39.18 & 22.68 & 14.72 & 9.63 & 0.09 \\
Cardiotocography & 0.03 & 0.15 & 0.08 & 0.24 & 0.51 & 0 & 0.1 & 0.38 & 0.02 & 0.02 & 3.42 & 0.04 & 0.17 & 0.42 & 4.49 & 9.1 & 17.24 & 32.9 & 28.43 & 45.55 & 17.88 & 16.16 & 11.19 & 0.09 \\
celeba & 0.1 & 3.65 & 1.52 & 0.29 & 27.96 & 0.02 & 1.99 & 4.8 & 0.07 & 0.1 & 19.79 & 0.07 & 0.65 & 1.45 & 23.77 & 35.23 & 77.89 & 160.95 & 127.82 & 162.91 & 35.39 & 97.78 & 69.31 & 0.32 \\
census & 4.52 & 18.94 & 1.99 & 0.65 & 316.96 & 0.17 & 15.79 & 5.65 & N/A & N/A & 33.12 & 0.1 & 5.7 & 3.22 & 25.35 & 35.22 & 72.15 & 167.27 & 127.84 & 172.04 & 46.93 & 104.73 & 86.05 & 1.15 \\
cover & 0.04 & 3.21 & 0.37 & 0.34 & 12.37 & 0.01 & 0.38 & 4.04 & 0.05 & 0.07 & 32.81 & 0.06 & 0.38 & 1.58 & 22.42 & 31.74 & 89.71 & 154.97 & 140.79 & 332.33 & 43.72 & 97.32 & 67.16 & 0.28 \\
donors & 0.04 & 2.96 & 0.29 & 0.14 & 11.03 & 0 & 0.33 & 3.74 & 0.02 & 0.04 & 34.5 & 0.06 & 0.4 & 1.49 & 21.07 & 28.95 & 88.05 & 157.72 & 126.86 & 302.71 & 37.32 & 99.51 & 72.77 & 0.24 \\
fault & 0.06 & 0.17 & 0.06 & 0.26 & 0.59 & 0.01 & 0.08 & 0.33 & 0.02 & 0.04 & 6.49 & 0.04 & 0.18 & 0.4 & 4.57 & 6.97 & 12.59 & 30.13 & 28.15 & 34.38 & 21.32 & 16.45 & 17.83 & 0.09 \\
fraud & 0.22 & 4.34 & 1.75 & 0.42 & 30.76 & 0.01 & 1.69 & 4.81 & 0.15 & 0.25 & 34.85 & 0.06 & 0.46 & 1.53 & 24.06 & 33.54 & 82.52 & 155.65 & 129.69 & 173.66 & 53.94 & 99.55 & 82.73 & 0.32 \\
glass & 0 & 0.04 & 0.01 & 0.2 & 0.17 & 0 & 0.01 & 0.12 & 0.02 & 0 & 2.47 & 0.01 & 0.13 & 0.2 & 2.45 & 2.63 & 8.54 & 15.38 & 11.36 & 40.52 & 30.88 & 7.03 & 8.35 & 0.06 \\
Hepatitis & 0.03 & 0.06 & 0.03 & 0.2 & 0.18 & 0.01 & 0.04 & 0.13 & 0 & 0.01 & 2.63 & 0.02 & 0.14 & 0.25 & 2.31 & 2.66 & 9.18 & 15.69 & 13.79 & 20.85 & 30.08 & 7.49 & 8.39 & 0.06 \\
http & 0 & 2.93 & 0.11 & 0.19 & 11.14 & 0 & 0.16 & 3.93 & 0.02 & 0.02 & N/A & 0.05 & 0.35 & 1.56 & 22.28 & 28.06 & 91.43 & 157.66 & 100.9 & 453.11 & 35.85 & 100.82 & 66.36 & 0.26 \\
InternetAds & 15.68 & 1.86 & 0.14 & 0.46 & 40.37 & 0.2 & 1.34 & 0.75 & 0.56 & 0.79 & N/A & 0.07 & 2.91 & 1.41 & 4.66 & 6.83 & 13.49 & 36.62 & 19.31 & 44.88 & 27.59 & 17.66 & 14.24 & 0.61 \\
Ionosphere & 0.08 & 0.04 & 0.03 & 0.19 & 0.23 & 0 & 0.04 & 0.14 & 0.02 & 0.02 & 2.6 & 0.03 & 0.15 & 0.31 & 1.64 & 2.54 & 6.21 & 16.4 & 9.95 & 19.09 & 39.06 & 6.89 & 6.56 & 0.07 \\
landsat & 0.28 & 1.57 & 0.69 & 0.4 & 15.3 & 0.01 & 0.73 & 2.2 & 0.1 & 0.16 & 21.95 & 0.05 & 0.38 & 1.13 & 11.72 & 19.31 & 40.87 & 101.35 & 60.07 & 111.3 & 56.6 & 59.29 & 49.12 & 0.25 \\
letter & 0.13 & 0.08 & 0.04 & 0.2 & 0.49 & 0 & 0.07 & 0.25 & 0.02 & 0.03 & 5.21 & 0.05 & 0.16 & 0.37 & 3.21 & 6.66 & 15.51 & 25.22 & 17.78 & 28.14 & 18.55 & 12.48 & 13.57 & 0.1 \\
Lymphography & 0.03 & 0.04 & 0.01 & 0.17 & 0.19 & 0 & 0.03 & 0.13 & 0.01 & 0.01 & 2.61 & 0.02 & 0.14 & 0.26 & 3.01 & 3.05 & 6.59 & 16.28 & 11.54 & 26.35 & 35.15 & 7.07 & 7.15 & 0.07 \\
magic.gamma & 0.04 & 3.25 & 0.23 & 0.39 & 12 & 0 & 0.35 & 6.42 & 0.04 & 0.09 & 32.91 & 0.07 & 0.37 & 1.58 & 22.45 & 28.68 & 72.77 & 158.06 & 79.75 & 313.71 & 35.67 & 98.46 & 71.75 & 0.3 \\
mammography & 0.01 & 2.96 & 0.19 & 0.25 & 10.54 & 0.01 & 0.3 & 9.15 & 0.01 & 0.04 & N/A & 0.06 & 0.31 & 1.89 & 24.58 & 22.05 & 75.61 & 156.66 & 82.74 & 356.13 & 36.61 & 101.65 & 42.04 & 0.26 \\
mnist & 0.8 & 2.76 & 1.08 & 0.6 & 45.1 & 0.04 & 1.53 & 5.11 & N/A & N/A & 26.33 & 0.06 & 0.86 & 1.88 & 18.45 & 23.18 & 49.98 & 119.6 & 86.85 & 42.63 & 41.35 & 70.54 & 42.47 & 0.34 \\
musk & 2 & 0.59 & 0.18 & 0.28 & 11.97 & 0.05 & 0.4 & 1.29 & 0.24 & 0.38 & 12.16 & 0.05 & 0.55 & 0.81 & 7.42 & 11.06 & 24.35 & 48.8 & 29.44 & 18.96 & 33.18 & 26.78 & 13.22 & 0.25 \\
optdigits & 0.56 & 1.12 & 0.48 & 0.34 & 12.85 & 0.04 & 0.76 & 3.09 & N/A & N/A & 17.09 & 0.05 & 0.41 & 1 & 13.01 & 20.67 & 45.81 & 82.42 & 49.19 & 56.27 & 63.18 & 48 & 24.04 & 0.25 \\
PageBlocks & 0.01 & 0.86 & 0.1 & 0.32 & 3.61 & 0.01 & 0.14 & 3 & 0.02 & 0.04 & 18.06 & 0.05 & 0.22 & 1.14 & 12.55 & 17.38 & 43.56 & 84.95 & 47.57 & 170.22 & 45.68 & 48.41 & 21.55 & 0.18 \\
pendigits & 0.09 & 1.87 & 0.76 & 0.36 & 6.22 & 0.01 & 0.86 & 4.86 & 0.04 & 0.07 & 22.57 & 0.06 & 0.33 & 1.5 & 11.58 & 23.89 & 55.66 & 107.62 & 57.65 & 180.92 & 65.4 & 65.48 & 28.44 & 0.21 \\
Pima & 0 & 0.04 & 0.01 & 0.17 & 0.18 & 0.01 & 0.02 & 0.2 & 0 & 0.01 & 2.54 & 0.02 & 0.13 & 0.25 & 1.82 & 2.83 & 8.54 & 15.73 & 10.12 & 35.39 & 48.23 & 7.54 & 4.05 & 0.08 \\
satellite & 0.19 & 1.57 & 0.73 & 0.47 & 14.12 & 0.02 & 0.79 & 4.69 & 0.09 & 0.17 & 14.59 & 0.07 & 0.38 & 1.59 & 14.85 & 19.47 & 41.83 & 101.04 & 60.93 & 101.16 & 73.17 & 55.4 & 26.49 & 0.25 \\
satimage-2 & 0.22 & 1.37 & 0.53 & 0.38 & 10.45 & 0.01 & 0.63 & 3.76 & 0.09 & 0.15 & 10.13 & 0.05 & 0.35 & 1.56 & 12.73 & 15.79 & 38.08 & 90.22 & 58.17 & 109.39 & 49.93 & 52.77 & 23.59 & 0.23 \\
shuttle & 0.03 & 3.16 & 0.17 & 0.22 & 11.07 & 0.01 & 0.25 & 8.75 & 0.03 & 0.05 & 18 & 0.06 & 0.36 & 2.19 & 21.56 & 31.35 & 68.32 & 157.14 & 86.32 & 319.05 & 40.19 & 100.2 & 43.96 & 0.28 \\
skin & 0 & 2.52 & 0.05 & 0.2 & 9.74 & 0 & 0.17 & 5.95 & 0.02 & 0.02 & N/A & 0.05 & 0.36 & 2.13 & 22.15 & 30.12 & 77.24 & 157.23 & 97.78 & 447.26 & 54.79 & 97.65 & 48.23 & 0.25 \\
smtp & 0 & 2.75 & 0.09 & 0.22 & 10.09 & 0 & 0.16 & 6.19 & 0.01 & 0.02 & 17.43 & 0.04 & 0.32 & 1.98 & 23.1 & 26.45 & 72.19 & 157.84 & 94.37 & 452.46 & 53.7 & 100.14 & 47.7 & 0.26 \\
SpamBase & 0.34 & 0.76 & 0.31 & 0.31 & 6.33 & 0.02 & 0.46 & 2.36 & 0.06 & 0.09 & N/A & 0.05 & 0.29 & 0.98 & 8.2 & 13.41 & 35.16 & 66.97 & 33.38 & 43.95 & 49.75 & 37.78 & 19.38 & 0.21 \\
speech & 3.98 & 1.43 & 0.26 & 0.56 & 30.86 & 0.15 & 1.39 & 2.19 & 0.71 & 0.91 & 8.3 & 0.06 & 1.5 & 1.1 & 7.38 & 13.35 & 33.69 & 60.37 & 38.33 & 26.18 & 36.01 & 29.67 & 19.96 & 0.42 \\
Stamps & 0 & 0.04 & 0.02 & 0.21 & 0.17 & 0 & 0.03 & 0.22 & 0 & 0 & 1.32 & 0.02 & 0.13 & 0.35 & 2.16 & 3.22 & 9.14 & 16.08 & 11.81 & 36 & 31.08 & 6.64 & 4.47 & 0.06 \\
thyroid & 0.01 & 0.42 & 0.07 & 0.35 & 1.58 & 0 & 0.1 & 1.31 & 0.02 & 0.02 & 6.5 & 0.04 & 0.19 & 0.81 & 8.36 & 12.36 & 33.78 & 59.73 & 35.09 & 161.78 & 33.33 & 29.38 & 15.38 & 0.14 \\
vertebral & 0 & 0.04 & 0.01 & 0.21 & 0.15 & 0.01 & 0.02 & 0.21 & 0 & 0.01 & 1.34 & 0.01 & 0.13 & 0.33 & 2.13 & 2.44 & 8.08 & 15.42 & 8.57 & 44.49 & 48.32 & 6.76 & 4.05 & 0.06 \\
vowels & 0.04 & 0.07 & 0.04 & 0.27 & 0.31 & 0.01 & 0.05 & 0.38 & 0.01 & 0.02 & 2 & 0.02 & 0.15 & 0.47 & 3.57 & 5.04 & 11.3 & 23.11 & 13.17 & 54.84 & 18.23 & 11.6 & 6.42 & 0.08 \\
Waveform & 0.16 & 0.42 & 0.21 & 0.39 & 1.57 & 0 & 0.24 & 1.59 & 0.03 & 0.06 & 5.98 & 0.05 & 0.22 & 0.94 & 7.18 & 12.31 & 29.66 & 54.34 & 32.54 & 80.1 & 34.76 & 30.02 & 14.91 & 0.16 \\
WBC & 0 & 0.02 & 0.02 & 0.16 & 0.17 & 0.01 & 0.02 & 0.24 & 0.01 & 0 & N/A & 0.02 & 0.13 & 0.3 & 2.14 & 2.54 & 8.2 & 15.17 & 8.61 & 35.47 & 41.99 & 7.66 & 4.59 & 0.06 \\
WDBC & 0.04 & 0.04 & 0.03 & 0.17 & 0.19 & 0.01 & 0.04 & 0.26 & 0.02 & 0.02 & 1.33 & 0.03 & 0.14 & 0.44 & 1.65 & 2.58 & 7.85 & 15.62 & 8.27 & 15.86 & 30.51 & 7.65 & 4.5 & 0.07 \\
Wilt & 0 & 0.56 & 0.06 & 0.36 & 2.59 & 0 & 0.09 & 2.01 & 0.01 & 0.02 & 8.21 & 0.03 & 0.21 & 0.94 & 10.22 & 14.01 & 34.23 & 76.74 & 38.37 & 201.36 & 43.17 & 40.04 & 22.43 & 0.17 \\
wine & 0.01 & 0.04 & 0.03 & 0.15 & 0.18 & 0.01 & 0.03 & 0.25 & 0.01 & 0.01 & 1.32 & 0.01 & 0.13 & 0.36 & 2.37 & 2.1 & 6.66 & 16.78 & 7.85 & 34.21 & 36.71 & 7.95 & 5.87 & 0.06 \\
WPBC & 0.06 & 0.04 & 0.03 & 0.22 & 0.18 & 0 & 0.05 & 0.26 & 0.01 & 0.03 & 1.34 & 0.03 & 0.14 & 0.45 & 2.47 & 2.26 & 6.23 & 15.73 & 7.99 & 15.8 & 48.7 & 7.8 & 5.99 & 0.08 \\
yeast & 0 & 0.06 & 0.03 & 0.22 & 0.27 & 0 & 0.04 & 0.42 & 0.01 & 0 & 2.58 & 0.02 & 0.14 & 0.45 & 2.73 & 4.38 & 12.32 & 23.11 & 12.16 & 56.32 & 17.26 & 12.13 & 9.16 & 0.07 \\
\bottomrule
\end{tabular}
\end{adjustbox}
\label{tab:appendix_runtime}
\end{sidewaystable}

\section{Reproducibility / Code Availability}

\label{app:code}

We value the availability and reproducibility of our work. The code and all the hyperparameters used in the experiment section are supplied as part of the supplemental material. We will also make our code publicly available upon acceptance of the paper.

\end{document}